\documentclass[sigconf,screen,nonacm]{acmart}
\AtBeginDocument{%
  }

\setcopyright{none}
\renewcommand\footnotetextcopyrightpermission[1]{}
\usepackage{cleveref}
\usepackage{hyperref}
\usepackage{booktabs}
\usepackage{threeparttable}
\usepackage[table]{xcolor}
\usepackage{array}
\usepackage{graphicx}
\usepackage{caption}
\usepackage{pifont}
\usepackage{listings}
\usepackage{xcolor}
\usepackage{multirow}

\definecolor{faintgreen}{RGB}{230,244,234}
\newcommand{\cmark}{\ding{51}}
\newcommand{\xmark}{\ding{55}}

\usepackage{booktabs}
\usepackage[table]{xcolor}

\definecolor{faintgreen}{RGB}{230,244,234}
\newcommand{\faintmidrule}{%
  \arrayrulecolor{black!20}\midrule
  \arrayrulecolor{black}%
}
\usepackage{amsthm}
\newtheoremstyle{nodefindent}% name
  {2pt}% Space above
  {2pt}% Space below
  {\itshape}% Body font
  {0pt}% Indent amount
  {\bfseries\itshape}% Head font
  {.}% Punctuation after head
  {0.5em}% Space after head
  {}% Head spec

\theoremstyle{nodefindent}
\newtheorem{definition}{Definition}[section]
\newtheorem{proposition}{Proposition}[section]
\begin{document}


\title{Score-Control for Hallucination Reduction in Diffusion Models}

%%
%% Author information for the arXiv version.
\author{Mahesh Bhosale\textsuperscript{*}, Naresh Kumar Devulapally\textsuperscript{*}, Abdul Wasi, Chau Pham, Vishnu Suresh Lokhande, David Doermann}
\affiliation{%
  \institution{University at Buffalo}
  \city{Buffalo}
  \state{NY}
  \country{USA}}

\renewcommand{\shortauthors}{Bhosale et al.}
\newcommand{\mahesh}[2]{%
    {\color{blue} #1}% The text being discussed
    {\color{cyan}\textbf{ [Mahesh: #2]}}% Your comment tag
}
\begin{abstract}
Diffusion models have emerged as the backbone of modern generative AI, powering advances in vision, language, audio and other modalities. Despite their success, they suffer from \emph{hallucinations}, implausible samples that lie outside the support of true data distribution, which degrade reliability and trust. In this work, we first empirically confirm previously proposed hypothesis that score smoothness causes hallucinations in \emph{Image Generation diffusion models} and provide a density-based perspective. We further formalize this notion by linking the hallucinations probability mass to lipschitz constant of the learned score function. Motivated by this, we introduce a \emph{Variance-Guided Score Modulation} (VSM) strategy that controls the score Jacobian, in turn reducing score smoothness and better approximating the ground truth score that decreases hallucinations. Empirical results on synthetic and real-world datasets demonstrate that our approach reduces hallucinations (up to $\sim$25\%) while maintaining high fidelity and diversity, providing a principled step toward more reliable diffusion-based image generation. We also propose two benchmark datasets with extreme semantic variation for systematic hallucination evaluation. Code and Datasets are publicly available at \url{https://github.com/bhosalems/VSM}.
\end{abstract}

\maketitle
\begingroup
\renewcommand{\thefootnote}{*}
\footnotetext{Equal contribution.}
\endgroup

\section{Introduction}

Diffusion models~\cite{ddim,ldm,ho2020denoising,nichol2021improved} have become the de facto backbone of multi-modality generation. They have been widely used in image synthesis~\cite{ldm, saharia2022palette}, audio generation~\cite{kushwaha2025diff}, text synthesis~\cite{wu2023ar,li2022diffusion}, and biomedical applications \cite{guo2024diffusion,bhosale2025pathdiff}. Recent text-to-image systems including Stable Diffusion 3.5~\cite{stability_sd35_2024} have pushed fidelity, controllability, and latency, enabling interactive editing. Adoption is accelerating at scale: within the span of two years, Adobe Firefly reports 22B+ assets generated as of April 2025 \cite{adobe_firefly_22b_2025}, and enterprise AI usage broadly rose to $78\%$ of organizations in 2024 \cite{stanford_ai_index_2025}.

While diffusion-based text-to-image systems are widely adopted, they raise well-documented concerns around fairness/bias, content safety, privacy, and copyright issues~\cite{trust_gen, hao2023safety, devulapally2025textencoderobjectlevelwatermarking, shen2023finetuning}. In this work we focus on \emph{hallucinations}: implausible samples generated by diffusion models (e.g., images of human hands with extra or missing fingers)~\cite{mode_interpolation, adapt_attention}. 

Beyond reducing sample quality, hallucinations undermine trust in the reliability of model generations. However, hallucinations in diffusion models are still largely underexplored. \cite{kim2024tacklingstructuralhallucinationimage} mitigate structural hallucination in image translation with multiple local diffusion. However, they do not use common text-conditioned image generation setup. \cite{mode_interpolation} study hallucination as mode interpolation but the work does not propose any hallucination mitigation strategies. \cite{adapt_attention} proposes to use temperature scaled self attention, but do not propose mitigation in text-conditioned image generation setting. In this work, we formalize a density-based view of hallucinations and introduce a simple, training-time method to reduce hallucinations during image generation.

Our \textbf{key contributions} are: 
(i) We establish a theoretical connection between score-field smoothness and hallucinations by deriving a lower bound on the learned model density at off-manifold points, showing that off-support probability mass remains non-zero and is governed by the score magnitude bound and its Lipschitz constant. This formalizes why overly smooth learned scores lead to hallucinated samples. (ii) Motivated by this result, we introduce \emph{Variance-Guided Score Modulation (VSM)}, an architecture-agnostic training objective that suppresses hallucinations by counteracting excessively smooth scores. VSM encourages higher local score curvature through a Jacobian-based smoothness penalty, and we derive a tractable diagonal approximation using the variance-learning parameterization of I-DDPM. We further apply this regularization with a time-dependent schedule that emphasizes late denoising steps where hallucinations are most likely to emerge.
(iii) We propose two datasets (\emph{ChessImages}, \emph{Cards}) with very large number $(\sim10^{44})$ of semantic classes to probe hallucination in controlled settings. Across multiple existing datasets, our method reduces hallucinations by up to $\sim\!26\%$, and on the proposed datasets by up to $\sim\!25\%$ compared to baselines.
\vspace{-1em}
\section{Related Work}

\textbf{Diffusion and Score-based models.}
Recently, diffusion models \cite{ho2020denoising,nichol2021improved,ddim,ldm} have gained prominence as a powerful approach for image generation, positioning themselves at the forefront of generative modeling techniques. Among these, denoising diffusion probability models (DDPMs)\cite{ho2020denoising} introduce a simple yet effective framework based on iterative noise removal, while variants such as DDIM\cite{ddim} improve sampling efficiency, enabling faster generation. Closely related are score-based generative models \cite{song2020score,song2019generative} that learn the gradient of the data distribution (score function) across noise levels and generate samples by solving stochastic or deterministic differential equations, offering improved flexibility and faster sampling. Furthermore, latent diffusion models (LDMs) \cite{ldm} improve efficiency by performing the diffusion process in a lower-dimensional latent space, significantly reducing computational costs while maintaining visual fidelity. However, many safety concerns are raised despite wide adoption of diffusion models. In this work we focus on mitigating hallucinations.

% Variance learning methods

\textbf{Hallucinations.} 
\cite{kim2024tacklingstructuralhallucinationimage} mitigates diffusion hallucinations via a local denoising pipeline over estimated OOD regions, but requires expert mask annotations for medical data. In contrast, VSM requires no additional annotations. \cite{mode_interpolation} introduce hallucinations as explained by mode interpolation: interpolating between disjoint modes due to smooth learned score approximations. But this work does not propose any mitigation technique. Oorloff et al.~\cite{adapt_attention} mitigate hallucinations by temperature scaling the self-attention softmax to suppress early-stage noise. \cite{lu2025towards} frame text hallucination as a local generation bias, introduce the Local Dependency Ratio (LDR) to measure it, and argue that stronger global dependencies help. However, their analysis is only focused on images containing text. \cite{wewer2025spatial} reduce hallucinations in structured reasoning via sequential generation with Spatial Reasoning Models (SRMs), but the approach is specialized to spatial reasoning and less applicable to general text-to-image generation. They also introduce MNIST Sudoku, whereas our ChessImages benchmark has a much larger semantic space ($\sim10^{44}$ vs.\ $\sim10^{22}$). DG~\cite{triaridis2025mitigatingdiffusionmodelhallucinations} mitigates diffusion hallucinations by dynamically selecting the classifier-guidance target at each denoising step to selectively sharpen hallucination-prone score directions during sampling. However, we observe this leads to mode collapse.

% \textcolor{red}{How our work is different.}
\section{Hallucinations in Diffusion Models} \label{sec:hallucination-defs}
We formalize hallucinations in the context of diffusion models \cite{mode_interpolation, pham2025memorization}. We categorize generated samples $\tilde{x} \sim \mathcal{P}_{\theta}$ into: (i) \textbf{Hallucinated} and (ii) \textbf{Non-Hallucinated} . We further sub-categorize non-hallucinated samples into (i) \textbf{Memorized} and (ii) \textbf{Generalized} samples. Let $\mathcal{P}_{\mathrm{data}}$ denote the unknown data distribution on $\mathcal{X}\subseteq\mathbb{R}^d$, and let $\mathcal{P}_{\theta}$ denote the model distribution induced by the diffusion model parameterized by $\theta$. We assume $\mathcal{P}_{\theta}$ admits a density $p_{\theta}$ with respect to Lebesgue measure on $\mathbb{R}^d$. When $\mathcal{P}_{\mathrm{data}}$ is absolutely continuous we denote its Lebesgue density by $p_{\mathrm{data}}$, and otherwise interpret $p_{\mathrm{data}}$ as an effective data density used to define low-density regions.
\noindent
\begin{definition}[\textbf{\textit{Hallucinated Samples}}]
\label{def:hallucination}
Formally, define the $\epsilon$-hallucination set as
\begin{equation}
\mathcal{H}_\epsilon
:= \big\{\, x \in \mathcal{X} : p_{\mathrm{data}}(x) \leq \epsilon \,\big\}
\label{eq:hallucination_definition}
\end{equation}
\end{definition}
A generated sample $\tilde{x}$ is \emph{hallucinated} if $\tilde{x} \in \mathcal{H}_\epsilon$.
Setting $\epsilon=0$ recovers samples that lie in regions where $p_{\mathrm{data}}(x)=0$.
For distributions with global support (e.g., Gaussian mixtures),
we instead choose a vanishingly small $\epsilon>0$ to define an \emph{effective support} and treat
samples in regions of negligible data density as hallucinations.
A sample is \emph{non-hallucinated} if $\tilde{x} \notin \mathcal{H}_\epsilon$.

\begin{definition}[\textbf{\textit{Memorization and Generalization Regions}}]
\label{def:memorized-generalized}
Let $d(\cdot,\cdot)$ denote a distance function on $\mathcal{X}$, and let $\delta > 0$ be a proximity threshold.
Given a training set $\mathcal{X}_{\mathrm{train}} = \{x^{(i)}\}_{i=1}^N$, define:

\noindent (i) Memorization region ($\mathcal{M}$):
\begin{equation}
    \mathcal{M} := \big\{\, x \in \mathcal{X}\setminus\mathcal{H}_\epsilon : \min_{i} d(x, x^{(i)}) \le \delta \,\big\}
    \label{eq:memorized_samples_definition}
\end{equation}

\noindent (ii) Generalization region ($\mathcal{G}$):

\vspace{-1em}

\begin{equation}
    \mathcal{G} := \mathcal{X} \setminus (\mathcal{H}_\epsilon \cup \mathcal{M})
    \label{eq:generalized_samples_definition}
\end{equation}

\end{definition}
A generated sample $\tilde{x} \sim \mathcal{P}_{\theta}$ is \emph{memorized} if $\tilde{x} \in \mathcal{M}$, and it is \emph{generalized} if $\tilde{x} \in \mathcal{G}$.
Throughout the paper, we treat $\epsilon$ and $\delta$ as fixed hyperparameters and omit the dependence of $\mathcal{M}$ and $\mathcal{G}$ on these hyperparameters in the notation for brevity.
By construction, $\mathcal{H}_\epsilon$, $\mathcal{M}$, and $\mathcal{G}$ are mutually exclusive and partition $\mathcal{X}$.
\noindent
\begin{definition}[\textbf{\textit{Hallucination Probability}}]
\label{def:hallucination-prob}
Having defined the region of the sample space that corresponds to hallucinations, we now quantify the likelihood of a model generating such samples.
The \emph{hallucination probability} $\mathbb{P}_{\theta}^{\mathrm{hall}}$ is defined as:

\vspace{-1em}

\begin{equation}
    \mathbb{P}_{\theta}^{\mathrm{hall}}(\epsilon)
:= \Pr_{\tilde{x} \sim \mathcal{P}_{\theta}}\!\left[\tilde{x} \in \mathcal{H}_\epsilon \right]
= \int_{\mathcal{H}_\epsilon} p_{\theta}(x)\, dx
\label{eq:prob_hallucination_definition}
\end{equation}

\end{definition}
In this work, we propose a method to reduce the incidence of hallucinated samples, i.e., samples falling in $\mathcal{H}_\epsilon$.
To assess potential side effects of hallucination mitigation, we further decompose the non-hallucinated region into the memorization and generalization regions, $\mathcal{M}$ and $\mathcal{G}$ (Definition~\ref{def:memorized-generalized}).
Since $\mathcal{H}_\epsilon$, $\mathcal{M}$, and $\mathcal{G}$ are mutually exclusive and partition $\mathcal{X}$, any shift in model behavior that decreases the probability of sampling from $\mathcal{H}_\epsilon$ must be reflected by a corresponding shift toward $\mathcal{M}$ and/or $\mathcal{G}$.
Therefore, our experiments report metrics that quantify both memorization and generalization.

\vspace{-0.5em}

\section{Methods}

In this section, we begin by establishing diffusion model preliminaries in~\cref{sec:preliminary}. In~\cref{sec:hallucinations-inevitable}, we confirm hallucinations are linked to score smoothness, providing theoretical motivation to control the score smoothness that we corroborate experimentally (\cref{fig:score_smoothness_motivation}). Finally, \cref{sec:vsm} introduces Variance-Guided Score Modulation (VSM), our proposed approach for mitigating the hallucinations. 
%Henceforth we use data measure $\mathcal{P}$ and density $p$ interchangeably in favor of simplicity.

\begin{figure*}[htbp]
  \centering
  \includegraphics[width=\linewidth]{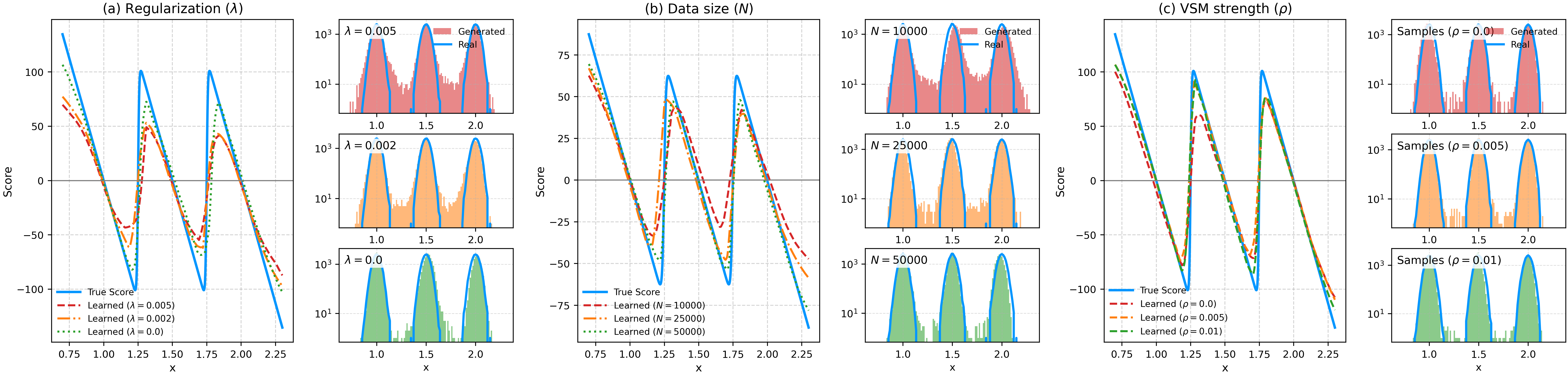}
  \caption{Motivation: Score smoothing causes hallucinations on mixture of 1D Gaussians with means $\mu=[1.0, 1.5, 2.0]$ and $\sigma=0.035$. We simulate score smoothness by adding weight normalization and changing training dataset size. a) increasing $\ell_2$ weight regularization ($\lambda$) on diffusion NN smoothens the learned score more increasingly leaking the probability mass in off-support regions causing more hallucinations. Seen as samples generated outside the support of the true data (represented by blue line). b) Decreasing training sample size also smoothens the score increasing hallucinations. c) increasing strength  ($\rho$) of VSM (our method) effectively reduces score smoothness, and reduces hallucinations.}
  \label{fig:score_smoothness_motivation}
\end{figure*}

% \vspace{-2em}

\subsection{Preliminaries.}
\label{sec:preliminary}
Let $\mathcal{X}\subseteq\mathbb{R}^d$ denote the data domain, and let $x_0 \sim \mathcal{P}_{\mathrm{data}}$ be a clean data sample. The \emph{score function}~\cite{song2020score} is given by: $s(x) \;=\; \nabla_x \log p(x).$ In the variance-preserving (VP) forward diffusion process~\cite{ho2020ddpm}, data are corrupted by Gaussian noise: $
q(x_t \mid x_0) = \mathcal{N}\!\big(\sqrt{\bar\alpha_t}\,x_0,(1-\bar\alpha_t)I\big),$ where $t\in\{1,\dots,T\}$ indexes the noise level and $\bar\alpha_t := \prod_{s=1}^t \alpha_s$. The ground-truth marginal score can be written as an expectation over conditional scores:
\begin{equation}
\label{eq:gt_score}
\begin{aligned}
s_{\mathrm{GT}}(x_t,t)
&= \nabla_{x_t} \log q_t(x_t) = \mathbb{E}_{x_0 \sim q(x_0 \mid x_t)}
\left[ \nabla_{x_t} \log q(x_t \mid x_0) \right] \\
&= \mathbb{E}_{x_0 \sim q(x_0 \mid x_t)}
\left[ -\frac{x_t - \sqrt{\bar\alpha_t}x_0}{1-\bar\alpha_t} \right]
\end{aligned}
\end{equation}
Where $q(x_0\mid x_t)$ is the posterior induced by the forward process. For a fixed $x_0$, the conditional score simplifies to
$ \nabla_{x_t}\log q_t(x_t\mid x_0) = -\epsilon/\sqrt{1-\bar\alpha_t}$.
Thus, the conditional score corresponds to the injected noise $\epsilon$ up to the scale factor $-(1-\bar\alpha_t)^{-1/2}$. In practice, the model $s_\theta(x_t,t)$ is trained to approximate the time-marginal score $\nabla_{x_t} \log q_t(x_t)$ \cite{song2020score}. Define the k-th dimension error for sample $i$ at noise level $t$ as:
\[
\Delta s_{k}^{(i)}(t) := s_{\theta,k}(x_t^{(i)},t) - s_{\mathrm{GT},k}(x_t^{(i)},t)
\]
for $k\in\{1,\dots,d\}$. We summarize the overall error via the root-mean-squared deviation:
\begin{equation}
\Delta s_{RMSE}
:= \sqrt{\frac{1}{NTd}\sum_{i=1}^{N}\sum_{t=1}^{T}\sum_{k=1}^{d}\big(\Delta s_{k}^{(i)}(t)\big)^2}
\label{eq:s_delta}
\end{equation}
$\Delta s_{RMSE}$ captures how well $s_\theta$ approximates the ground-truth score field, we use it to empirically validate its relationship with hallucinations on image datasets.  

\begin{table*}[t]
  \centering
  \fontfamily{ptm}\selectfont
  \scriptsize
  \setlength{\tabcolsep}{6pt}
  \renewcommand{\arraystretch}{1.05}
  \captionsetup{skip=0.3em}
  \resizebox{\textwidth}{!}{%
  \begin{tabular}{@{}l l c c r r@{}}
    \toprule
    \textbf{Dataset} &
    \textbf{Detection Type} &
    \textbf{Time (100 images)} &
    \textbf{RGB} &
    \textbf{Size} &
    \textbf{Semantic Classes} \\
    \midrule
    1D, 2D & Six-Sigma thresholding & $\sim$2 s & \xmark & $10^5$ & low ($\leq$ 25 modes) \\
    Hands \cite{hands11k} & Human annotation & $\sim$12 min & \cmark & 11{,}079 & low ($\leq$ 10) \\
    Shapes \cite{mode_interpolation} & Training-free rules & $\sim$2.5 s & \xmark & 22{,}000 & low (3) \\
    MNIST \cite{mnist} & Classifier thresholding & $\sim$4 s & \xmark & 60{,}000 & low (10) \\
    ImageNet-1K \cite{russakovsky2015imagenetlargescalevisual} & Improved Precision, Recall & $\sim$2 min & \cmark & >500k & High ($1000$) \\
    \textbf{Cards (proposed)} & Training-free rules & $\sim$2.5 s & \cmark & 94{,}000 & Very High ($10^5$) \\
    \textbf{ChessImages (proposed)} & Training-free rules & $\sim$2.5 s & \cmark & 190{,}000 & Extreme ($\geq$ $10^{44}$) \\
    \bottomrule
  \end{tabular}%
  }
  \caption{Datasets used. A \emph{semantic class} denotes a valid, interpretable configuration. The proposed \textit{Cards} and \textit{ChessImages} feature vast semantic spaces and allow rapid training-free hallucination detection, making them effective benchmarks for systematic hallucination studies.}
  \label{tab:dataset_charact}
\end{table*}

\subsection{Motivation}
\label{sec:hallucinations-inevitable}
Diffusion models learn an approximate score function that is a smoothed version of the sharp ground-truth score field (\cref{fig:score_smoothness_motivation}), which \cite{mode_interpolation} identifies as a cause of hallucinations. To confirm this hypothesis empirically, we control the degree of smoothing through weight regularization and training dataset size, and observe its effect on the number of hallucinations in a 1D Gaussian dataset. Specifically, we consider a 1D Gaussian mixture with component means $\{1.0,\,1.5,\,2.0\}$ and shared standard deviation $0.35$. For regularization, we add $\ell_2$ weight regularization to the neural network trained to predict the added noise~\cite{interpolation}. This can be viewed as limiting the network's capacity to represent complex score functions. As shown in the left part of Fig.~\ref{fig:score_smoothness_motivation}a, increasing the regularization strength $\lambda$ increases the smoothness of the learned score. In the right part of Fig.~\ref{fig:score_smoothness_motivation}, we sample points from the model and observe that this increased smoothness leads to more hallucinations, measured as generated samples that fall between modes (outside the $6\sigma$ effective support of the Gaussian mixture, indicated by the blue boundary). This suggests that the model-implied density decays more slowly than the ground-truth density, yielding non-negligible probability mass in low-density regions, even when $\lambda=0$. Similarly, in Fig.~\ref{fig:score_smoothness_motivation}b, we observe that decreasing the dataset size increases score smoothness, leading to more hallucinations. We also observe a positive correlation ($R^2 = 0.44$) between hallucinations and the score error $\Delta s_{\mathrm{RMSE}}$ on the Hands dataset (see Appendix), confirming that this effect extends beyond the simple 1D Gaussian setting. We further formalize the relationship between hallucinations and score smoothness in \cref{prop:gap-lower-bound1}. In Fig.~\ref{fig:score_smoothness_motivation}c, we show that our proposed method, VSM, effectively reduces score smoothness matching the sharp ground-truth score better and thereby decreases the incidence of hallucinations.

\begin{proposition}[\textbf{\textit{Relationship Between Score Smoothness and Hallucinations}}]\label{prop:gap-lower-bound1}
For an off-manifold point $x$ at distance $\delta_x$ from a high-density region of data, the model density admits the lower bound:
\vspace{-1em}
$$p_\theta(x) \ge C_b \exp\left( -S \delta_x - \frac{L}{2} \delta_x^2 \right) > 0$$

\vspace{-1em}

\noindent Where, $C_b>0$ denotes a minimum model density value on the boundary of the high-density region of data, while $L$ and $S$ denote the Lipschitz constant of the learned score field and an upper bound on its magnitude, respectively. The Lipschitz constant L is defined as, $L = \sup_{\Delta x \neq 0} \frac{\|s_\theta(x + \Delta x) - s_\theta(x)\|}{\|\Delta x\|}.$ We assume standard regularity conditions on $s_\theta$ (see the Appendix for proof and more details).
\end{proposition}

\noindent \textbf{\textit{Takeaway:}} \Cref{prop:gap-lower-bound1} formalizes the intuition that hallucinations arise from score-field smoothness: if a learned score is too smooth (low $L$), probability mass is forced to leak exponentially into off-manifold regions, creating implausible generated samples. In~\cref{sec:vsm}, we propose a way to control $L$ thereby reducing off-support probability mass leakage.

Notably, hallucination probability diminishes as the distance from the manifold $\delta_x$ increases. This is consistent with the experimental findings of \cite{mode_interpolation}, who demonstrate that increasing the separation between 1D Gaussian modes leads to a measurable reduction in number of hallucinations. While ~\cref{prop:gap-lower-bound1} identifies the score’s Lipschitz constant $L$ as the primary driver of off-manifold mass, precisely localizing these boundaries requires knowledge of the data distribution's support in high dimensional score field \textit{apriori} which is not known. Instead, our work focuses on globally modulating the Lipschitz regularity of the learned score to increase its local curvature, thereby suppressing the off-manifold density leakage that drives hallucination mass. We experimentally confirm that global application is helpful to reduce the hallucinations (\cref{tab:finetune_table_all}) across all the dataset. 

% \mahesh{Latent-space analysis via $\beta$-VAE reveals that VSM increases score sharpness within off-manifold gap regions, confirming its efficacy in targeting hallucination-prone zones without explicit boundary supervision (\cref{fig:beta_vae_analysis})}{I think this is risky, because Jacobian would be higher in the gap region which means the penalty is low, where around high density regions with similar logic penalty would be high, so I think it is better not to show them the the scores in the main manuscript. There is some mechanism that despite counterintutitive penalty is causing reduction in hallucinations but I dont know what that mechanism is}.  

\subsection{Variance Guided Score Modulation}
\label{sec:vsm}
As established in~\cref{prop:gap-lower-bound1}, hallucination mass is driven by smooth score (learned score’s small Lipschitz constant $L$ as compared to ground truth score). Since, score Jacobian $J_\theta$ satisfies,
\[
\|J_\theta(x)\|_2 \le L \quad \forall x,
\]
encouraging larger Jacobian magnitudes can implicitly increase the effective Lipschitz constant of the score field. Therefore, we define smoothness penalty as, 
\begin{equation}
\label{eq:smoothness_penalty}
\mathcal{L}_{\mathrm{VSM}}
\;=\;
\mathbb{E}_{t,x_t}\!\Big[\;\phi\!\enspace\big(\,\|J_\theta(x_t,t)\|^2\,\big)\;\Big],
\qquad
\phi(u)=\frac{1}{u+\eta},\ \eta>0,
\end{equation}

\noindent\textbf{Tractability.}
Computing the full Jacobian of the high-dimensional marginal score
$s_\theta(x_t,t)$ is intractable.
We therefore use a \emph{diagonal curvature proxy} derived from variance learning.
Following I-DDPM~\cite{nichol2021improved}, we parameterize the reverse transition as
$p_\theta(x_{t-1}\!\mid x_t)=\mathcal{N}\!\big(\mu_\theta(x_t,t),\Sigma_\theta(x_t,t)\big)$ and optimize the variational objective $\mathcal{L}_{\mathrm{VLB}}$ in~\cref{eq:vlb}
to learn a diagonal approximation of the reverse conditional covariance,
$\Sigma_\theta(x_t,t)\approx \mathrm{diag}\!\big(\sigma_\theta^2(x_t,t)\big)$.
This yields a diagonal precision matrix
$\Sigma_\theta(x_t,t)^{-1}\approx \mathrm{diag}\!\big(1/\sigma_\theta^2(x_t,t)\big)$.
Note that for Gaussian noising kernel
$q(x_t\!\mid x_{t-1})=\mathcal{N}(a_t x_{t-1},\sigma_t^2 I)$, Bayes' rule gives,
\begin{equation}
\begin{aligned}
% \log p(x_{t-1}\mid x_t)
% &= \log p_{t-1}(x_{t-1}) + \log q(x_t\mid x_{t-1}) - \log p_t(x_t), \\
\nabla_{x_{t-1}}^{2}\log p(x_{t-1}\mid x_t)
&= \nabla_{x_{t-1}}^{2}\log p_{t-1}(x_{t-1})
+ \nabla_{x_{t-1}}^{2}\log q(x_t\mid x_{t-1})\\
&= \nabla_{x_{t-1}}^{2}\log p_{t-1}(x_{t-1}) - \frac{a_t^{2}}{\sigma_t^{2}}I.
\end{aligned}
\end{equation}
The key consequence of the above decomposition is that the only \emph{tractable} curvature contribution comes from the Gaussian kernel (second term), while the remaining marginal curvature is captured by the reverse conditional covariance learned via variance prediction (LHS). With a local Gaussian approximation of the marginal~\cite{MengSongLiErmon2021, alger2024point},
$\nabla_{x_{t-1}}^2\log p_{t-1}(x_{t-1})\approx -\Sigma_{t-1}^{-1}$,
and using the learned reverse conditional curvature
$\nabla_{x_{t-1}}^2 \log p_\theta(x_{t-1}\mid x_t)=-\Sigma_\theta(x_t,t)^{-1}$,
\[
\nabla_{x_{t-1}}^2\log p_{t-1}(x_{t-1})
\approx
-\Sigma_\theta(x_t,t)^{-1}
+
\frac{a_t^2}{\sigma_t^2}I.
\]
We retain only the sample-dependent diagonal term and obtain practical diagonal proxy for curvature that we use in $\mathcal{L}_{\mathrm{VSM}}$,
\[
J_{\theta}(x_{t-1}, t-1) = 
\nabla_{x_{t-1}}s_\theta(x_{t-1},t-1) \approx\
\mathrm{diag}(-\,1/\sigma_\theta^2(x_t,t)).
\]

\begin{table*}[t]
  \centering
  \fontfamily{ptm}\selectfont
  \scriptsize
  \setlength{\tabcolsep}{4pt}
  \renewcommand{\arraystretch}{1.05}
  \captionsetup{skip=0.3em}
  \resizebox{\textwidth}{!}{%
    \begin{tabular}{@{}l cc cc cc@{}}
      \toprule
      \textbf{Method} &
      \multicolumn{2}{c}{\textbf{1D Gaussian}} &
      \multicolumn{2}{c}{\textbf{2D Gaussian}} &
      \multicolumn{2}{c}{\textbf{Hands-11K}} \\
      \cmidrule(lr){2-3}\cmidrule(lr){4-5}\cmidrule(lr){6-7}
      & \textbf{Score RMSE}$\downarrow$ & \textbf{H\%}$\downarrow$ ($\times 10^{-3}$)
      & \textbf{Score RMSE}$\downarrow$ & \textbf{H\%}$\downarrow$
      & \textbf{Score RMSE}$\downarrow$ & \textbf{H\%}$\downarrow$ \\
      \midrule
      DDPM\textsuperscript{\textdagger} &
      10.5573 $\pm$ 0.0115 & 5.2173 $\pm$ 1.92 &
      19.60 $\pm$ 0.0242 & 1.1844 $\pm$ 0.0108 &
      21.92 $\pm$ 0.57 & 11.00 $\pm$ 2.37 \\
      \rowcolor{faintgreen}
      + VSM $\mathcal{L}_{\mathrm{VSM}}$ &
      \textbf{7.7645 $\pm$ 0.0141} & \textbf{2.7027 $\pm$ 0.863} &
      \textbf{18.70 $\pm$ 0.0888} & \textbf{1.0831 $\pm$ 0.00823} &
      \textbf{15.49 $\pm$ 0.29} & \textbf{5.01 $\pm$ 1.98} \\
      \bottomrule
    \end{tabular}%
  }
  \caption{Score RMSE and hallucination rate across synthetic Gaussian mixtures (1D/2D) and Hands-11K. \textit{Across all datasets, VSM reduces score error, thereby reducing hallucinations}.}
  \label{tab:score_error_table}
\end{table*}

\begin{table*}[t]
  \centering
  \fontfamily{ptm}\selectfont
  \footnotesize
  \setlength{\tabcolsep}{3pt}
  \renewcommand{\arraystretch}{0.95}
  \captionsetup{skip=0.3em}
  \begin{threeparttable}
  \resizebox{\textwidth}{!}{%
    \begin{tabular}{@{}>{\raggedright\arraybackslash\hspace{0pt}}p{2.55cm}
                    *{4}{>{\centering\arraybackslash}p{1.25cm}}|
                    *{4}{>{\centering\arraybackslash}p{1.25cm}}@{}}
      \toprule
      \textbf{Method}
        & \multicolumn{4}{c|}{\textbf{Hands-11K}}
        & \multicolumn{4}{c}{\textbf{MNIST}} \\
      \cmidrule(lr){2-5}\cmidrule(l){6-9}
        & \textbf{C-FID} \(\downarrow\) & \textbf{FID} \(\downarrow\) & \textbf{FLD} \(\downarrow\) & \textbf{H\%} \(\downarrow\)
        & \textbf{C-FID} \(\downarrow\) & \textbf{FID} \(\downarrow\) & \textbf{FLD} \(\downarrow\) & \textbf{H\%} \(\downarrow\) \\
      \midrule
      DDPM \cite{nichol2021improved}
        & 12.00 & 126.25 & 35.99 & 23.33
        & 16.23 & 112.16 & 28.14 & 4.50 \\
      \rowcolor{faintgreen}
      + VSM $\mathcal{L}_{\mathrm{VSM}}$
        & {10.13} & {108.12} & {22.20} & \textbf{{5.15}}
        & {8.47} & {43.75} & \underline{6.99} & {3.50} \\
      \faintmidrule
      LDM-UC \cite{ldm}
        & 8.89 & 45.78 & 24.87 & 19.66
        & 11.82 & 76.98 & 25.29 & \underline{1.83} \\
      \rowcolor{faintgreen}
      + VSM $\mathcal{L}_{\mathrm{VSM}}$
        & {\underline{7.75}} & \textbf{{43.98}} & {22.21} & {16.54}
        & \textbf{{3.91}} & \underline{31.38} & \textbf{{6.28}} & \textbf{{0.33}} \\
      \faintmidrule
      LDM-Text Cond. \cite{ldm}
        & 10.02 & 83.96 & \underline{21.34} & 29.50
        & {8.89} & 230.13 & 23.59 & 23.00 \\
      \rowcolor{faintgreen}
      + VSM $\mathcal{L}_{\mathrm{VSM}}$
        & \textbf{5.58} & {44.95} & \textbf{{20.07}} & {21.15}
        & 9.36 & {228.21} & {8.74} & {12.48} \\
      \faintmidrule
      LDM-PT \cite{mahajan2024prompting}
        & 10.17 & \underline{44.15} & 24.20 & 24.83
        & \underline{8.44} & 64.27 & 23.58 & 19.83 \\
      AAM\textsuperscript{\textdagger\textdagger} \cite{adapt_attention}
        & -- & 102.30 & -- & \underline{9.20}
        & -- & \textbf{15.10}  & -- & 5.70 \\
      \midrule

      \textbf{Method}
        & \multicolumn{4}{c|}{\textbf{Cards}}
        & \multicolumn{4}{c}{\textbf{Shapes}} \\
      \cmidrule(lr){2-5}\cmidrule(l){6-9}
        & \textbf{C-FID} \(\downarrow\) & \textbf{FID} \(\downarrow\) & \textbf{FLD} \(\downarrow\) & \textbf{H\%} \(\downarrow\)
        & \textbf{C-FID} \(\downarrow\) & \textbf{FID} \(\downarrow\) & \textbf{FLD} \(\downarrow\) & \textbf{H\%} \(\downarrow\) \\
      \midrule
      DDPM \cite{nichol2021improved}
        & 9.10 & 112.33 & 33.29 & 22.41
        & 26.07 & 123.34 & 21.84 & 29.50 \\
      \rowcolor{faintgreen}
      + VSM $\mathcal{L}_{\mathrm{VSM}}$
        & \textbf{{2.20}} & \underline{64.35} & \underline{21.40} & \textbf{{2.33}}
        & {18.98} & {98.61} & {17.29} & \textbf{{3.00}} \\
      \faintmidrule
      LDM-UC \cite{ldm}
        & 7.28 & 87.53 & 42.54 & 17.60
        & \underline{2.04} & \underline{24.42} & \underline{9.74} & 7.17 \\
      \rowcolor{faintgreen}
      + VSM $\mathcal{L}_{\mathrm{VSM}}$
        & \underline{3.78} & \textbf{{32.54}} & \textbf{{19.35}} & \underline{7.60}
        & \textbf{{1.56}} & \textbf{{19.84}} & \textbf{{7.04}} & \underline{4.67} \\
      \midrule

      \textbf{Method}
        & \multicolumn{4}{c|}{\textbf{ChessImages}}
        & \multicolumn{4}{c}{\textbf{ImageNet-1K}} \\
      \cmidrule(lr){2-5}\cmidrule(l){6-9}
        & \textbf{C-FID} \(\downarrow\) & \textbf{FID} \(\downarrow\) & \textbf{FLD} \(\downarrow\) & \textbf{H\%} \(\downarrow\)
        & \textbf{C-Pre.} \(\uparrow\) & \textbf{C-Rec.} \(\uparrow\) & \textbf{FLD} \(\downarrow\) & \textbf{FID} \(\downarrow\) \\
      \midrule
      DDPM \cite{nichol2021improved}
        & 3.74 & 191.68 & 96.83 & 71.00
        & 0.44 & 0.18 & 19.19 & 135.57 \\
      \rowcolor{faintgreen}
      + VSM $\mathcal{L}_{\mathrm{VSM}}$
        & 4.32 & 191.19 & \textbf{48.75} & 56.01
        & {0.63} & \underline{0.43} & {15.95} & {126.32} \\
      \midrule
      LDM-UC \cite{ldm}
        & \underline{3.59} & \underline{29.15} & 89.65 & \underline{11.66}
        & 0.56 & 0.41 & \underline{7.23} & \underline{76.86} \\
      \rowcolor{faintgreen}
      + VSM $\mathcal{L}_{\mathrm{VSM}}$
        & \textbf{3.54} & \textbf{34.67} & \underline{52.17} & \textbf{9.28}
        & \underline{0.68} & {\textbf{0.51}} & \textbf{{4.77}} & \textbf{{69.97}} \\
        DG \textsuperscript{\textdagger\textdagger}\cite{triaridis2025mitigatingdiffusionmodelhallucinations}
        & -- & -- & -- & --
        & \textbf{0.75} & 0.23 & -- & -- \\
      \bottomrule
    \end{tabular}%
    }
  \caption{VSM reduces hallucinations relative to baselines across Hands-11K, MNIST, Cards, Shapes, ChessImages, and ImageNet-1K. Metrics: C-FID = CLIP-FID, FID = Inception-FID, FLD \cite{fld}, H\% = hallucination rate, CLIP-Prec./Rec. = improved precision/recall in CLIP feature space. ``--'' indicates metrics not reported. Bold is used to represent best and underline for the second best result. $\textsuperscript{\textdagger\textdagger}$ represents numbers from ArXiV, public code unavailable.}
  \label{tab:finetune_table_all}
  \end{threeparttable}
\end{table*}

\noindent\textbf{Training objective.}
We augment the standard denoising noise-matching objective $\mathcal{L}_\mathrm{DM}$ with the variational term for variance learning $\mathcal{L}_\mathrm{VLB}$ and smoothness penalty $\mathcal{L}_\mathrm{VSM}$ (\cref{eq:smoothness_penalty}):
\begin{equation}
\label{eq:loss_total}
\mathcal{L}_{\mathrm{Total}}
= \mathcal{L}_{\mathrm{DM}} + \mathcal{L}_{\mathrm{VLB}} + \eta(t)\,\mathcal{L}_{\mathrm{VSM}},
\end{equation}
where,
\begin{equation}
\mathcal{L}_{\mathrm{DM}}
= \mathbb{E}_{x_0,\epsilon,t}\big[\|\epsilon-\epsilon_\theta(x_t,t)\|^2\big]
\label{eq:loss_total}
\end{equation}
is the standard noise-prediction loss that equivalently learns the marginal score field. We adopt the I-DDPM variational objective~\cite{nichol2021improved} as $\mathcal{L}_{\mathrm{VLB}}$,

\vspace{-1em}

\begin{align}
\label{eq:vlb}
\mathcal{L}_{\mathrm{VLB}}
&:= \mathcal{L}_0 + \sum_{t=2}^{T}\mathcal{L}_{t-1} + \mathcal{L}_T, \\
\mathcal{L}_0
&:= - \log p_\theta(x_0 \mid x_1), \\
\mathcal{L}_{t-1}
&:= D_{\mathrm{KL}}\!\big(q(x_{t-1}\mid x_t, x_0)\ \Vert\ p_\theta(x_{t-1}\mid x_t)\big), \qquad t=2,\dots,T, \\
\mathcal{L}_T
&:= D_{\mathrm{KL}}\!\big(q(x_T\mid x_0)\ \Vert\ p(x_T)\big).
\end{align}
Where $q(x_{t-1}\mid x_t,x_0)$ is the closed-form forward posterior. The KL terms $\mathcal{L}_{t-1}$ provide direct supervision for learning $\Sigma_\theta$ by matching $p_\theta$ to $q$ at each timestep. The I-DDPM framework further facilitates the implementation of $\mathcal{L}_{\mathrm{VSM}}$ through efficient fine-tuning. By adding a variance-learning head to a pre-trained checkpoint, we avoid the prohibitive cost of training from scratch. Our experiments compare the efficacy of this fine-tuning approach for smoothness correction against usual finetuning.

\noindent \textbf{Time dependent scaling.} Since hallucinations tend to emerge during the late stages of sampling~\cite{mode_interpolation, adapt_attention}, we use a time-varying weighting
\begin{equation}
\label{eq:time-dependent-equation}
\eta(t) = \frac{\rho}{\sqrt{1 - \bar{\alpha}_t}}
\end{equation}
where $\rho$ is a tunable hyperparameter. This schedule progressively increases the VSM penalty as sampling approaches the low-noise regime, thereby punishing smoothing near the final denoising steps while avoiding the overly aggressive weighting of a fully inverse schedule.
\begin{figure*}[t]
  \centering
  \includegraphics[width=\textwidth]{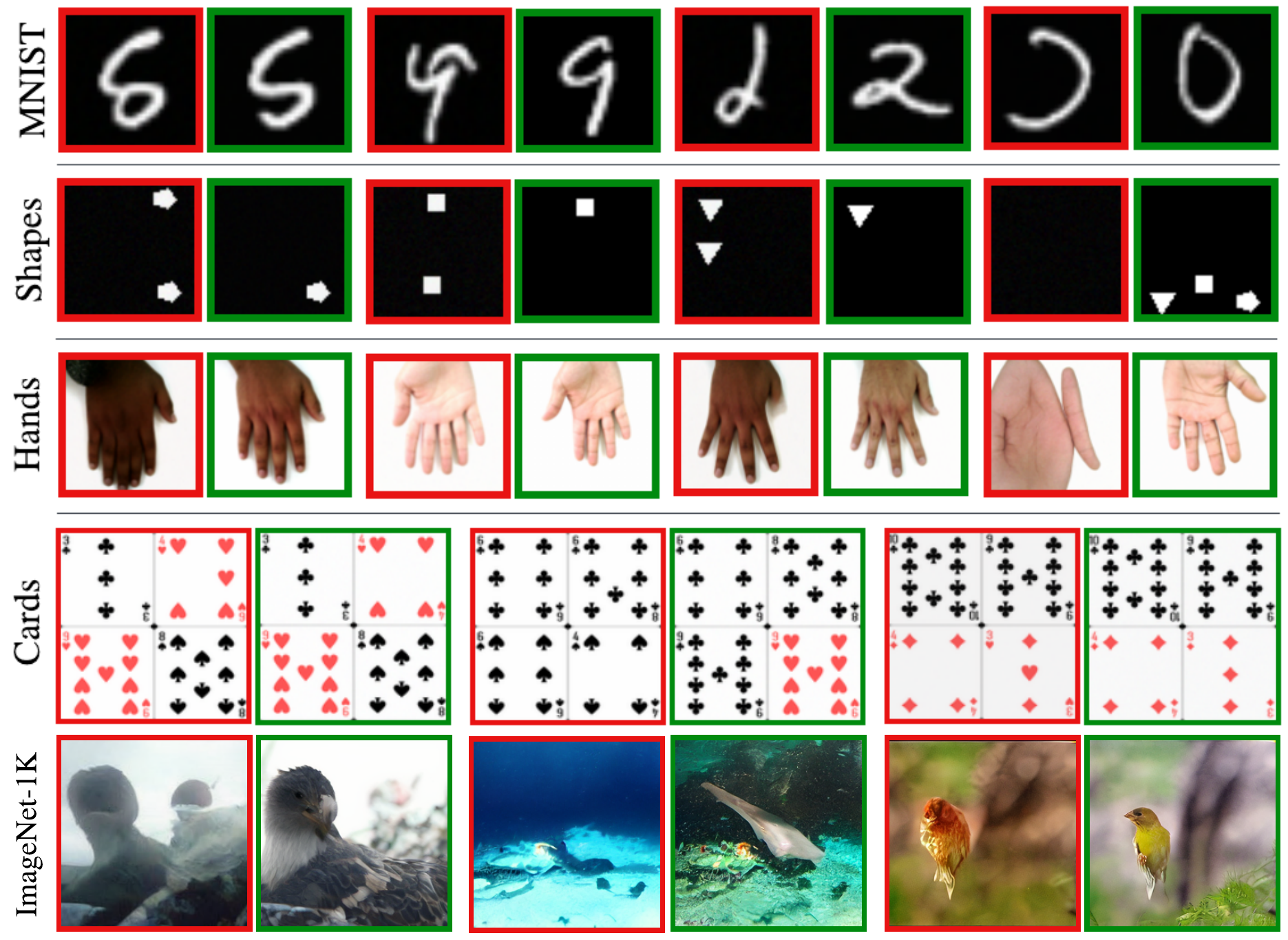}
  \caption{Qualitative examples of corrected hallucinations with VSM. Each pair shows hallucinated generations (red) and corrected valid generations (green) across datasets.}
  \label{fig:qual_fig}
\end{figure*}

\section{Experiments}
\label{sec:experiments}

% -------------------------------------------------------------------------

\subsection{Hallucination Detection Module}
\label{sec:hall_detect}

To operationalize Definition~\ref{def:hallucination}, we introduce hallucination detection modules $\mathcal{D} : \mathcal{X} \to \{0,1\}$ for each dataset that classify each generated sample $\tilde{x} \sim \mathcal{P}_{\theta}$ as hallucinated ($\mathcal{D}(\tilde{x}) = 1$) or non-hallucinated ($\mathcal{D}(\tilde{x}) = 0$). We consider four instantiations: (i) \textbf{human annotation} (Hands-11K), (ii) \textbf{classifier thresholding} (MNIST), (iii) \textbf{training-free rule/validator checks} (Shapes, Cards, ChessImages), and (iv) \textbf{improved precision and recall for real-world datasets} (ImageNet-1K). We calibrate $\mathcal{D}$ such that for all real samples  $\Pr[\mathcal{D}(x)=1 \mid x \sim \mathcal{P}_{\mathrm{data}}]=0$, ensuring that detected hallucinations primarily reflect implausible generations from $\mathcal{P}_\theta$ rather than detector bias.

% -------------------------------------------------------------------------

\begin{table*}[t]
  \centering
  \fontfamily{ptm}\selectfont
  \footnotesize
  \setlength{\tabcolsep}{0pt}
  \renewcommand{\arraystretch}{0.9}
  \captionsetup{skip=0.3em}
  \resizebox{\textwidth}{!}{%
    \begin{tabular}{@{}p{2.9cm}
                    >{\centering\arraybackslash}p{1.45cm}
                    >{\centering\arraybackslash}p{1.45cm}
                    >{\centering\arraybackslash}p{1.45cm}
                    >{\centering\arraybackslash}p{1.45cm}
                    >{\centering\arraybackslash}p{1.45cm}
                    >{\centering\arraybackslash}p{1.45cm}
                    >{\centering\arraybackslash}p{1.45cm}
                    >{\centering\arraybackslash}p{1.45cm}@{}}
      \toprule
      \textbf{Method} &
      \multicolumn{4}{c}{\textbf{Hands-11K}} &
      \multicolumn{4}{c}{\textbf{MNIST}} \\
      \cmidrule(lr){2-5}\cmidrule(lr){6-9}
      & \textbf{C-FID}$\downarrow$ & \textbf{FID}$\downarrow$ & \textbf{FLD}$\downarrow$ & \textbf{H\%}$\downarrow$
      & \textbf{C-FID}$\downarrow$ & \textbf{FID}$\downarrow$ & \textbf{FLD}$\downarrow$ & \textbf{H\%}$\downarrow$ \\
      \midrule
      Fine-tune LDM & \textbf{4.46} & 57.56 & 18.47 & 25.66 & 9.73 & 50.61 & 17.53 & 0.31 \\
      \rowcolor{faintgreen}
      Fine-tune LDM + VSM & 4.98 & \textbf{53.24} & \textbf{16.64} & \textbf{19.66} & \textbf{3.81} & \textbf{31.27} & \textbf{4.42} & \textbf{0.24} \\
      \midrule
      \textbf{Method} &
      \multicolumn{4}{c}{\textbf{ChessImages}} &
      \multicolumn{4}{c}{\textbf{ImageNet-1K}} \\
      \cmidrule(lr){2-5}\cmidrule(lr){6-9}
      & \textbf{C-FID}$\downarrow$ & \textbf{FID}$\downarrow$ & \textbf{FLD}$\downarrow$ & \textbf{H\%}$\downarrow$
      & \textbf{CLIP-Prec.}$\uparrow$ & \textbf{CLIP-Rec.}$\uparrow$ & \textbf{FLD}$\downarrow$ & \textbf{FID}$\downarrow$ \\
      \midrule
      Fine-tune LDM & \textbf{3.73} & 42.49 & 73.92 & 17.50 & 0.58 & 0.15 & 60.44 & 73.51 \\
      \rowcolor{faintgreen}
      Fine-tune LDM + VSM & 4.16 & \textbf{23.50} & \textbf{48.84} & \textbf{15.66} & \textbf{0.71} & \textbf{0.48} & \textbf{52.77} & \textbf{62.27} \\
      \bottomrule
    \end{tabular}%
  }
  \caption{Variance-head-only fine-tuning. Adding VSM during fine-tuning reduces hallucinations compared to fine-tuning without VSM, while preserving quality metrics. These results suggest that VSM can serve as an effective corrective mechanism for pretrained checkpoints.}
  \label{tab:finetune_table}
\end{table*}

\begin{figure*}[t]
  \centering
  \includegraphics[width=\textwidth]{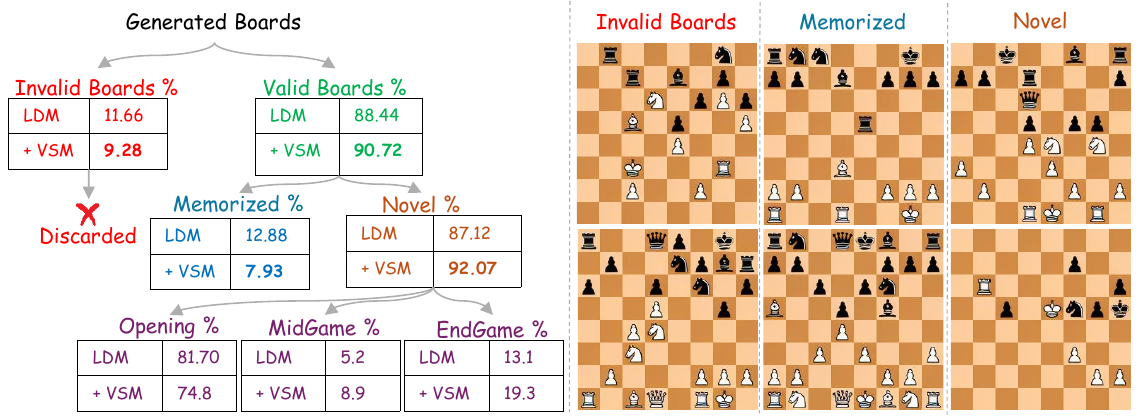}
  \caption{Categorization of generated chessboards into invalid (hallucinated), memorized (seen in train), and generalized (novel) samples. VSM consistently outperforms the LDM baseline.}
  \label{fig:buckets}
\end{figure*}

\vspace{-1em}

\subsection{Experimental Setup}
\label{sec:experimental-setup}

\subsubsection{\textbf{Datasets}}
We evaluate on both synthetic and real-world image datasets, and additionally propose two novel datasets designed for systematic hallucination analysis. A key dataset attribute is the number of \textit{semantic classes}, i.e., the set of structurally valid configurations/categories each sample can belong to. Constrained datasets (e.g., MNIST with 10 digits) offer limited class spaces, whereas combinatorial datasets (e.g., the proposed ChessImages dataset with $\sim10^{44}$ valid board states) exhibit extreme diversity, making hallucinations easier to surface. Our proposed \textit{Cards} and \textit{ChessImages} datasets combine extremely large semantic spaces with efficient training-free validators, enabling large-scale hallucination studies.

\subparagraph{\textbf{\textit{Datasets with simple semantic classes:}}} \noindent {{\textit{(i) 1D and 2D Gaussian mixtures:}}}
For 1D, we sample from a three-mode Gaussian mixture with means $\{1.0,1.5,2.0\}$ and $\sigma=0.035$.
For 2D, we use a $5\times5$ grid of Gaussians with $\sigma=0.02$.
Following \cite{mode_interpolation}, the data support is defined as $\pm6\sigma$ around each mean; samples outside are labeled hallucinated.
We train a diffusion denoising model on $5\times 10^4$ data points and draw $10^6$ samples at inference. \noindent \textit{(ii) MNIST:}
MNIST consists of $28\times28$ grayscale digit images (0-9) \cite{mnist}. A CNN trained on MNIST (99.5\%+ test acc.) flags outputs with confidence below $0.98$. \noindent \textit{(iii) Shapes:} Shapes contains $64\times64$ images split into three vertical regions, each assigned a square, pentagon, or triangle \cite{alaa2022faithful}. Valid images have at most one shape per region, yielding 6 semantic classes. Hallucinations include duplicates, missing shapes, or shapes in wrong regions. A template-matching pipeline is used as the hallucination detection module that achieves 100\% region-and-shape accuracy on real data. \noindent \textit{(iv) Hands:} Hands-11K contains $128\times128$ images of human hands with exactly five fingers \cite{hands11k}.
Hallucinations include missing/extra/malformed fingers.
Three human annotators identify hallucinations.
Semantic classes: $2$ orientations $\times2$ (palm up/down) $=4$.~\Cref{tab:dataset_charact} details dataset characteristics. 

\subparagraph{\textbf{\textit{Datasets with extreme semantic class spaces:}}} \textit{\textbf{(i) Cards (proposed):}} Synthetic $128\times128$ images arranged as a $2\times2$ grid of playing cards (Ace to 10), with standard templates from Wikipedia.\footnote{\url{https://en.wikipedia.org/wiki/Playing_card}}
A generation is hallucinated if symbol count mismatches value, color is incorrect, symbols are missing/invalid, or conflicting symbols appear.
Detection is completely automated via template matching (100\% accurate on the dataset). \textit{\textbf{(ii) ChessImages (proposed):}}
$256\times256$ chessboards rendered from FEN strings sampled from VALUED \cite{saha2023valued},
with standardized templates.\footnote{\url{https://en.wikipedia.org/wiki/Template:Chess_diagram}} We reconstruct FEN via template matching (100\% accurate on the real samples), then validate legality with \texttt{python-chess}. More details about \texttt{python-chess} and samples from proposed datasets are included in the Appendix. \textit{(iii) ImageNet-1K:}
We additionally evaluate on the real-world ImageNet-1K \cite{russakovsky2015imagenetlargescalevisual}, using the \emph{train split for training}. Since explicit hallucination detectors are not available at ImageNet scale, we evaluate improved precision and improved recall \cite{improvedprecision} in both Inception and CLIP feature spaces, together with FID (see \cref{tab:finetune_table_all}).

% \vspace{-1em}

\subsubsection{\textbf{Implementation Details:}} 
\subparagraph{\textbf{\textit{Models:}}} We test our method by integrating it with both pixel-space diffusion (DDPM~\cite{ho2020ddpm}) and latent diffusion (LDM \cite{ldm}). Within LDM, we report results for: (i) unconditional generation (LDM-UC), (ii) text-conditional generation (LDM-C), and (iii) conditional generation with prompt tuning (LDM-PT) \cite{mahajan2024prompting}. Prompt tuning details are provided in the Appendix. Where available, we also compare against the hallucination reduction baselines AAM \cite{adapt_attention}, and Dynamic Guidance \cite{triaridis2025mitigatingdiffusionmodelhallucinations}.
\subparagraph{\textbf{\textit{Training regimes:}}}
We evaluate two training regimes: (i) \textit{from-scratch training}, where the full model is trained from random initialization (~\cref{tab:finetune_table_all}), and (ii) \textit{variance-head-only training}, which mirrors the common practice of extending publicly available pretrained checkpoints to a target dataset. In this setting, we attach a variance head, to a pretrained checkpoint and subsequently fine-tune the model(~\cref{tab:finetune_table}).
\subparagraph{\textbf{\textit{Metrics:}}} On datasets with explicit hallucination detectors (Hands, MNIST, Shapes, Cards, ChessImages), we report:
(i) hallucination rate $H\%$ (lower is better),
(ii) FID (Inception features) and C-FID (CLIP features) which cpatures image fidelity,
and (iii) FLD \cite{fld} computed \emph{only on non-hallucinated samples} that measures fidelity, diversity, and novelty in feature space. For synthetic 1D/2D mixtures with closed form score, we report score error via $\Delta s_{RMSE}$.
On ImageNet-1K, following~\cite{triaridis2025mitigatingdiffusionmodelhallucinations} we use improved precision in CLIP feature space as a measure of hallucinations.

\vspace{-1.6em}

% -------------------------------------------------------------------------
\subsection{Results}
\label{sec:results}

\noindent \textbf{\textit{VSM reduces score error and hallucinations when support is measurable:}}
We first validate VSM in settings where hallucinations can be defined precisely and score error can be measured directly. On 1D and 2D Gaussian mixtures, where samples outside the effective data support are labeled hallucinated, VSM reduces both $\Delta s_{\mathrm{RMSE}}$ and the hallucination rate (\cref{tab:score_error_table}). This matches the intended effect of VSM: variance-guided score modulation better aligns the learned score with the ground-truth score and suppresses probability mass leakage into low-density regions. We observe the same trend on the higher-dimensional Hands-11K dataset, where VSM consistently lowers both score error and hallucination rate, showing that the mechanism extends beyond synthetic mixtures to real image data.

\noindent \textbf{\textit{VSM reduces hallucinations across diverse image datasets:}} 
We next evaluate on datasets with explicit hallucination detectors spanning both low-cardinality semantic spaces (MNIST, Shapes, Hands-11K) and large combinatorial spaces (Cards, ChessImages, ImageNet-1k). 
\Cref{tab:finetune_table_all} shows that VSM consistently reduces hallucination rate across both pixel-space diffusion (DDPM) and latent diffusion (LDM), and across conditioning regimes (unconditional, text-conditioned, prompt-tuned). 
Notably, hallucination reduction does not trade off against sample quality: in many cases VSM also improves fidelity metrics (FID/C-FID) and novelty/diversity as measured by FLD. 
On ChessImages, where legality is rule-checkable and the semantic space is extreme, VSM substantially reduces invalid boards while preserving visual structure, enabling controlled studies of validity under combinatorial constraints. Since explicit hallucination detectors are unavailable at ImageNet scale, we evaluate hallucination mitigation indirectly using CLIP-space precision and recall, together with FID, as reported in \cref{tab:finetune_table_all}. Higher precision indicates that a larger fraction of generated samples lies within the support of the real data distribution, whereas higher recall reflects better coverage of its modes. Relative to the baseline LDM-UC model, VSM improves both precision and recall, suggesting that it reduces off-support generations while simultaneously improving distributional coverage. In comparison to~\cite{triaridis2025mitigatingdiffusionmodelhallucinations}, which attains higher precision at the expense of a substantial drop in recall indicative of mode collapse, VSM achieves markedly stronger recall while maintaining a closely comparable precision. These results suggest that VSM mitigates hallucinations without sacrificing sample diversity.

\noindent \textbf{\textit{Qualitative results:}} 
\Cref{fig:qual_fig} shows representative hallucinations corrected by VSM. 
Across datasets, VSM suppresses off-manifold artifacts (e.g., invalid card symbols,malformed fingers), while preserving global structure and visual fidelity. These examples qualitatively align with the quantitative trend that VSM reduces invalid generations without introducing any artifacts.

% -------------------------------------------------------------------------
\vspace{-1.1em}

\subsection{Generalization vs Memorization in ChessImages}
\label{sec:gen-vs-mem}

The proposed ChessImages dataset enables analysis beyond hallucination rates because legality is rule-checkable and the semantic space is extremely large. After discarding invalid boards, we partition valid samples into \emph{memorized} boards that exactly match training positions and \emph{generalized} boards that are valid but unseen during training (Definition~\ref{def:memorized-generalized}). An ideal generator should increase the fraction of valid generations while shifting mass toward generalized positions. As shown in \cref{fig:buckets}, VSM outperforms the baseline LDM by reducing invalid generations and increasing the share of valid novel boards, enabling controlled evaluation of generalization in a rule-valid setting.

\vspace{-1em}

% -------------------------------------------------------------------------
\subsection{Variance-Head-Only Fine-tuning}
\label{sec:finetune}

% Our method (VSM) can be attached to pretrained diffusion checkpoints by training only the variance head and optimizing $\mathcal{L}_{\mathrm{VLB}} + \eta(t)\mathcal{L}_{\mathrm{VSM}}$.
% This enables a drop-in hallucination reduction mechanism for existing diffusion checkpoints without changing sampling-time compute.
Table~\ref{tab:finetune_table} reports variance-head-only fine-tuning results. Across datasets, incorporating VSM during fine-tuning consistently reduces hallucinations compared to fine-tuning without VSM, while preserving fidelity and diversity. These results suggest that VSM can serve as an effective corrective mechanism for pretrained checkpoints, offering a practical alternative to from-scratch training, which is often computationally expensive.

% -------------------------------------------------------------------------

\vspace{-1em}

\subsection{Ablation Studies}
\label{sec:ablations}
\begin{figure}[t]
    \centering
    \includegraphics[width=\columnwidth]{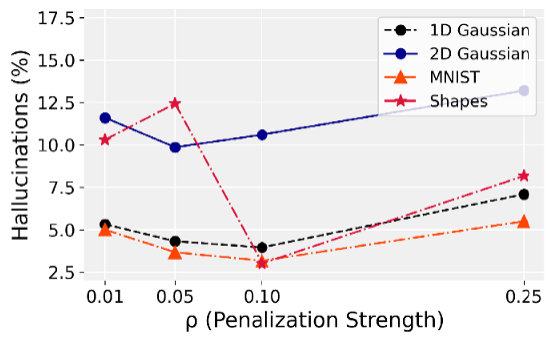}
    \caption{Increasing $\rho$ decreases hallucinations until it start increasing it back because diffusion loss gets excessively down-weighted causing suboptimal results. H\% for 1D and 2D are scaled by $10^3$ and $10^1$ respectively.}
    \label{fig:ablation_rho}
\end{figure}
We ablate (i) the regularizer strength $\rho$ check the effect of strength of VSM on hallucinations, and
(ii) the time-dependent scaling schedule $\eta(t)$ to assess the impact of late-stage emphasis. We observe increasing $\rho$ reduces hallucinations (H\%) by suppressing low-support mass, however after a point increased strength can overpower the diffusion loss, increasing hallucinations, therefore, we use a sweet spot $\rho=0.1$.

\begin{figure}[t]
  \centering
  \scriptsize
  \fontfamily{ptm}\selectfont
  \setlength{\tabcolsep}{3.2pt}
  \renewcommand{\arraystretch}{1.03}
  \captionsetup{skip=0.2em}

  \begin{minipage}[c]{0.56\columnwidth}
    \centering
    \vspace{0pt}
    \resizebox{\linewidth}{!}{%
      \begin{tabular}{@{}lccc@{}}
        \toprule
        \textbf{Schedule} & \textbf{C-FID}$\downarrow$ & \textbf{FLD}$\downarrow$ & \textbf{H\%}$\downarrow$ \\
        \midrule
        $\eta(t)=\rho(1-\bar{\alpha}_t)$ & 17.18 & 19.30 & 7.83 \\
        $\eta(t)=\rho/(1-\bar{\alpha}_t)$ & 11.05 & 7.61 & 5.00 \\
        $\eta(t)=\rho/\sqrt{1-\bar{\alpha}_t}$ & \textbf{3.91} & \textbf{6.99} & \textbf{3.50} \\
        \bottomrule
      \end{tabular}%
    }
  \end{minipage}\hfill
  \begin{minipage}[c]{0.40\columnwidth}
    \centering
    \vspace{0pt}
    \includegraphics[width=\linewidth]{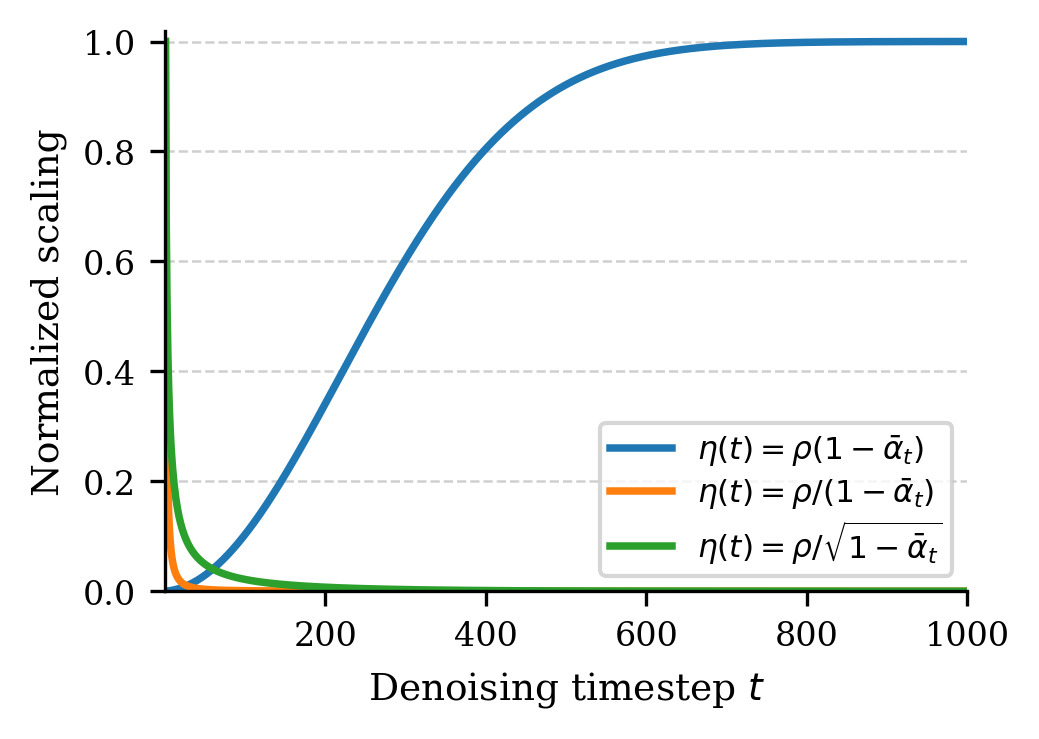}
  \end{minipage}

  \vspace{-0.4em}
  \caption{Ablation of time-dependent scaling schedules $\eta(t)$ on MNIST. The inverse square-root schedule achieves the lowest C-FID, FLD, and hallucination rate.}
  \label{fig:ablate_eta}
  \vspace{-0.6em}
\end{figure}

We ablate three choices for the time-dependent scaling schedule $\eta(t)$ on MNIST, as reported in \cref{fig:ablate_eta}. The results show that late-stage upweighting of the VSM penalty is important for suppressing hallucinations, but that overly aggressive scaling is suboptimal. In particular, the linear schedule $\eta(t)=\rho(1-\bar{\alpha}_t)$ performs worst across all metrics, indicating that weak late-stage (lower t) regularization is insufficient. The fully inverse schedule $\eta(t)=\rho/(1-\bar{\alpha}_t)$ improves substantially, but remains inferior to the inverse square-root form. Overall, $\eta(t)=\rho/\sqrt{1-\bar{\alpha}_t}$ achieves the best performance on C-FID, FLD, and hallucination rate, suggesting that moderate growth of the penalty toward the final denoising steps provides the best balance between preserving global structure and enforcing effective smoothing.

\section{Conclusion} 
% \vspace{-0.5em}
We present a density-based perspective on hallucinations in diffusion models, showing that excessive score smoothness causes probability mass to leak into off-support regions at an exponential rate controlled by the Lipschitz constant. Motivated by this insight, we introduce VSM, an architecture-agnostic method that increases the score Jacobian to suppress such leakage and thereby mitigate hallucinations. Extensive experiments on synthetic data, real-world datasets, and newly introduced challenge benchmarks show that VSM consistently reduces hallucinations while preserving high fidelity and diversity. More broadly, our work not only provides a practical and effective mitigation strategy, but also establishes a theoretical foundation for understanding hallucinations in diffusion models and contributes benchmark settings for their systematic evaluation.

\noindent{\bf Limitations:} Our approach is designed to mitigate hallucinations rather than eliminate them entirely. Additionally, a systematic understanding of hallucinations in natural image datasets, as well as reliable metrics to detect and quantify them in such domains, remains an open problem for future work.
\appendix

\twocolumn[
\begin{center}
    {\LARGE \textbf{Appendix}}
\end{center}
\vspace{1em}
]

\section{Towards Zero Hallucinations during generation}

% \vspace{-5mm}

\begin{figure}[h]
  \centering
  \includegraphics[width=0.99\linewidth]{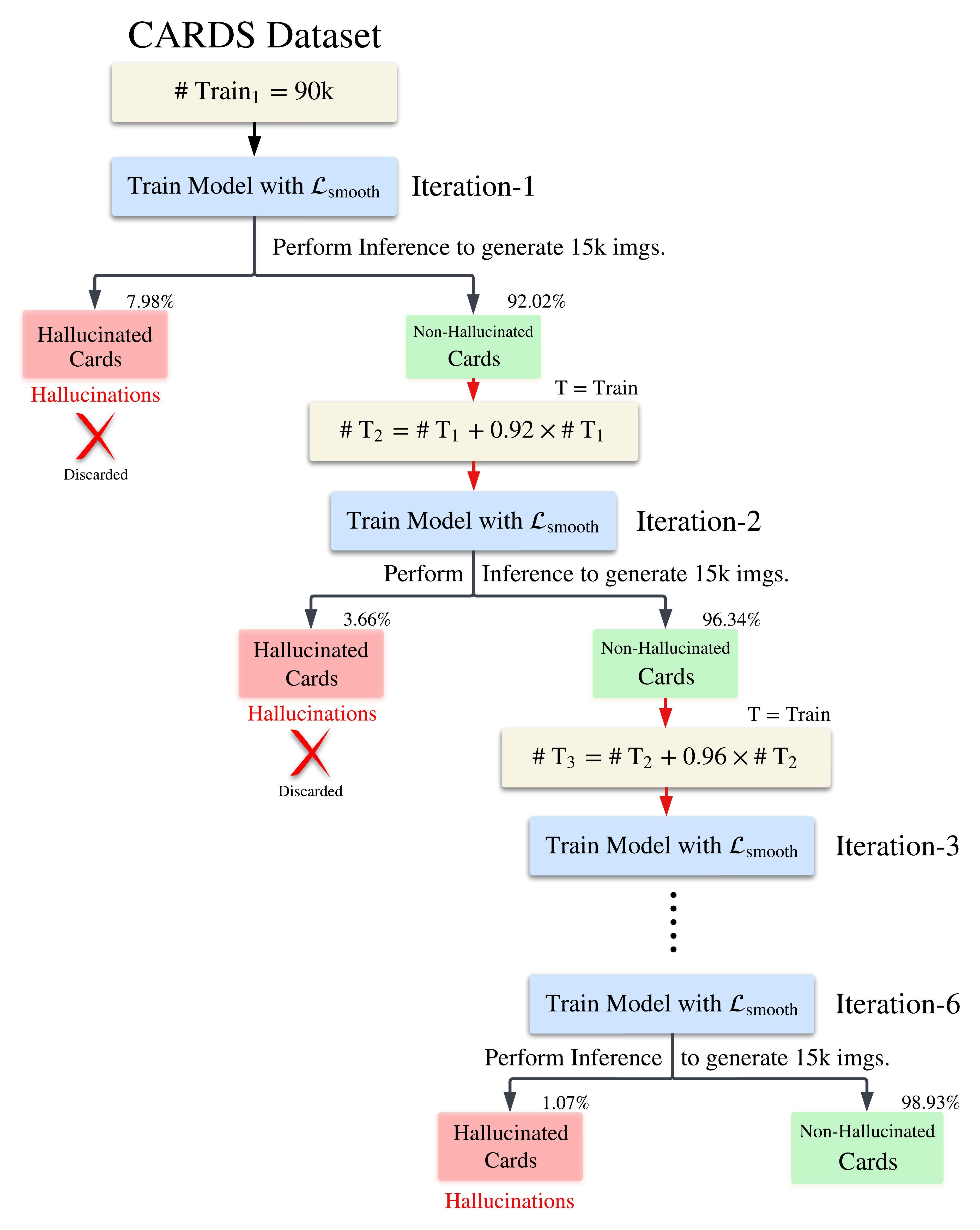}
  \caption{Iterative Training while appending Non-Hallucinated Images to $\mathcal{P}_{\text{train}}$}
  \label{fig:iterative_training}
  \vspace{-1em}
\end{figure}

\noindent
We propose a way that drives the hallucination rate toward zero. Figure~\ref{fig:iterative_training} illustrates the effectiveness of our proposed method ($\mathcal{L}_{\text{VSM}}$ loss) within an iterative training strategy to systematically reduce hallucinations during image generation. Beginning with a base model trained on an initial dataset of 90K images, each iteration involves generating 15,000 new card images, filtering out hallucinated samples, and appending only valid, non-hallucinated cards to the training set for the next iteration. This progressively refined dataset, denoted as $p_{\text{train}}$, is then used to retrain the model again from scratch. As shown, hallucination rates drop sharply from 7.98\% in iteration-1 to 1.07\% by iteration-6, while the proportion of non-hallucinated outputs steadily increases to 98.93\%. This iterative bootstrapping approach demonstrates how $\mathcal{L}_{\text{VSM}}$ enables the model to internalize valid patterns and avoid degenerate generations over time, leading to near-zero hallucinations during generation.

\begin{minipage}{0.4\linewidth}
  \centering
  \scalebox{0.95}{%
    \begin{tabular}{@{}l c@{}}
      \toprule
      \textbf{Iteration} & \textbf{Hal. Rate (\%)} \\
      \midrule
      Iteration-1 & 7.98 \\
      Iteration-2 & 3.66 \\
      Iteration-3 & 2.74 \\
      Iteration-4 & 1.82 \\
      Iteration-5 & 1.19 \\
      Iteration-6 & 1.07 \\
      \bottomrule
    \end{tabular}%
  }
\end{minipage}%
\hfill
\begin{minipage}{0.6\linewidth}
\captionof{table}{Reduction in hallucination rate over training iterations using the proposed $\mathcal{L}_{\text{VSM}}$ objective. As iterative training progresses, the rate of hallucinated generations decreases substantially, validating a way towards zero hallucinations.}
\label{tab:hallucination_reduction}
\end{minipage}

\section{More Details on Proposition 4.1}
\noindent\textbf{Setup.}
Following the Union of Manifolds Hypothesis (UMH)~\cite{umh}, we assume the
\emph{support} of the data measure $\mathcal{P}_{\mathrm{data}}$ admits a decomposition into a
disjoint union of the closures of $K$ connected, low-dimensional manifolds $\{\mathcal{A}_k\}_{k=1}^{K}$:
\[
\mathcal{S} = \operatorname{supp}(\mathcal{P}_{\mathrm{data}})
=
\bigsqcup_{k=1}^{K} \operatorname{cl}(\mathcal{A}_k)
\subseteq \mathbb{R}^d,
\qquad
\dim(\mathcal{A}_k)=d_k<d,
\]
where $\bigsqcup$ denotes a disjoint union and $\operatorname{cl}(\cdot)$ denotes closure in the ambient
space $\mathcal{X} \subseteq \mathbb{R}^d$. We define the off-support region as:
\[
\mathcal{O} := \mathbb{R}^d \setminus \mathcal{S}.
\]

By definition of support, $\mathcal{P}_{\mathrm{data}}(\mathcal{O})=0$, i.e., the data measure assigns
zero probability mass to the off-support region $\mathcal{O}$. For some radius $r>0$, define the tubular neighborhood:
\[
U := \{x \in \mathbb{R}^d : \operatorname{dist}(x,\mathcal{S}) \le r\}.
\]

\noindent \textbf{Regularity properties.} 
We utilize certain 
regularity properties of the learned diffusion density $p_\theta$ in tubular region $U$. We state following properties (P) and assumptions (A):
\begin{description}
\item[(P1)] \textbf{Positivity and continuity of the model density.}
The DDPM reverse process~\cite{ho2020ddpm} defines the generated distribution as,
\[
p_\theta(x_0)
\;=\;
\int p(x_T)\prod_{t=1}^{T}p_\theta(x_{t-1}\mid x_t)\,dx_{1:T},
\]
where $p(x_T)=\mathcal{N}(0,I)$ and each reverse transition
$p_\theta(x_{t-1}\mid x_t)=\mathcal{N}\!\big(\mu_\theta(x_t,t),\,\sigma_t^2 I\big)$
has non-degenerate covariance ($\sigma_t^2>0$).
Because every Gaussian component is strictly positive on $\mathbb{R}^d$,
the marginal $p_\theta(x_0)$, being a continuous mixture of such
Gaussians, is itself strictly positive and continuous on $\mathbb{R}^d$.
This is a standard property of convolutions with nondegenerate Gaussian
kernels~\cite{song2020sde,folland1999real}.
We further assume that the network parameterization $\mu_\theta(\cdot,t)$ is
smooth, which, combined with the smoothness of Gaussian convolutions,
ensures $p_\theta$ is differentiable on the region of interest so that the
score $s_\theta(x) =\nabla\log p_\theta(x)$ is well-defined wherever it is
used below.

\item[(A1)] \textbf{Compactness of the support and the tubular neighborhood.}
Since each $\mathcal{A}_k$ is bounded (a natural assumption for real-world data
residing in a finite region of $\mathbb{R}^d$ e.g. $\mathcal{X}=[0, 1]^d$ for images datasets), each closure
$\operatorname{cl}(\mathcal{A}_k)$ is compact. As a finite union of compact sets,
$\mathcal{S}$ is compact. For any $r>0$, the tubular neighborhood
$U=\{x\in\mathbb{R}^d:\operatorname{dist}(x,\mathcal{S})\le r\}$ is then
closed and bounded, hence also compact.

\item[(A2)] \textbf{Local Lipschitz score regularity and boundedness on $U$.}
We assume the score $s_\theta(x) =\nabla \log p_\theta(x)$ is $L$-Lipschitz on $U$, i.e.,
$\|s_\theta(x)-s_\theta(y)\|\le L\|x-y\|$ for all $x,y\in U$.
Since Lipschitz functions are continuous and $U$ is compact by~\textup{(A2)},
$s_\theta$ is bounded on $U$, and we define
\[
S \;:=\; \sup_{x\in U}\|s_\theta(x)\| \;<\;\infty.
\]
\emph{Remark.} The Lipschitz assumption is motivated by the smoothing
induced by Gaussian perturbations at nonzero noise
levels~\cite{song2020sde}. However, Lipschitz singularities may arise
near the zero-noise limit without additional
control~\cite{yang2024lipschitzsingularitiesdiffusionmodels}. We avoid
such cases by applying VSM for $t>0$.

\item[(P2)] \textbf{Boundary density lower bound.}
Since $p_\theta$ is continuous on $\mathbb{R}^d$~\textup{(P1)} and
$\mathcal{S}$ is compact~\textup{(A2)}, the extreme value
theorem~\cite{rudin1976principles} guarantees that $p_\theta$ attains its
minimum on $\mathcal{S}$. Moreover, since $p_\theta$ is strictly positive
on $\mathbb{R}^d$~\textup{(P1)}, this minimum is strictly positive:
\[
% \inf_{z\in \mathcal{S}} p_\theta(z)
% \;=\;
\min_{z\in \mathcal{S}} p_\theta(z)
\;=:\;
C_b
\;>\; 0.
\]
\end{description}

Together, properties (P1), (P2) and assumptions (A1), (A2) establish that diffusion models produce smooth, strictly positive, and well-behaved densities around the data support $\mathcal{S}$, making them amenable to the quantitative analysis in Proposition 4.1.
Specifically, the proof relies on only two genuine assumptions: compactness of the data support~\textup{(A1)} and Lipschitz regularity of the learned score~\textup{(A2)}, the remaining ingredients~\textup{(P1), (P2)} follow from the nondegenerate Gaussian structure of the DDPM reverse process.

\label{app:proof_gap_lower_bound}
\noindent\textit{\textbf{Proposition 4.1 (Relationship Between Score Smoothness and Hallucinations)}}
\textit{Let $x\in\mathcal{O}$ be an off-manifold point with
$\delta_x := \mathrm{dist}(x,\mathcal{S})\le r$, so that $x$ lies in the
tubular neighborhood $U$. Under~\textup{(P1)}, \textup{(A1)},
\textup{(A2)}, and~\textup{(P2)}, the model density admits the lower bound:
$$p_\theta(x) \;\ge\; C_b \exp\!\left( -S\,\delta_x - \frac{L}{2}\,\delta_x^2 \right) \;>\; 0.$$}
\begin{proof}
Since $\mathcal{S}$ is compact and nonempty~\textup{(A2)}, the continuous
function $z\mapsto \|x-z\|$ attains its minimum over $\mathcal{S}$. Let
$y\in\mathcal{S}$ be a minimizer. Then,
\[
\|x-y\|=\delta_x,
\qquad
y\in\mathcal{S}\subseteq U,
\qquad
\text{and}\qquad
\delta_x\le r \;\Rightarrow\; x\in U.
\]
Moreover, for any $t\in[0,1]$ the point $z_t:=y+t(x-y)$ satisfies
$\operatorname{dist}(z_t,\mathcal{S})\le\|z_t-y\|=t\,\delta_x\le r$,
so the entire segment $[y,x]$ lies in $U$.

Define $f(z):=\log p_\theta(z)$. By~\textup{(P1)}, $f$ is differentiable
on $U$ with gradient $\nabla f(z)=s_\theta(z)$, which is $L$-Lipschitz
on $U$ by~\textup{(A3)}. By Taylor's theorem with integral remainder,
\[
f(x)
\;=\;
f(y)+\langle \nabla f(y),\,x-y\rangle
+\int_0^1\!\langle \nabla f(y+t(x{-}y))-\nabla f(y),\;x-y\rangle\,dt.
\]
The integral term is bounded below using Cauchy-Schwarz~\cite{rudin1976principles} and the
$L$-Lipschitz property of $\nabla f$ on $U$ (noting that the segment
$[y,x]\subseteq U$):
\[
\int_0^1\!\langle \nabla f(y+t(x{-}y))-\nabla f(y),\;x-y\rangle\,dt
\;\ge\;
-\int_0^1 Lt\,\|x-y\|^2\,dt
\;=\;
-\tfrac{L}{2}\,\delta_x^2.
\]
Therefore,
\begin{equation}
\label{eq:quadratic_lower_bound}
f(x)\;\ge\; f(y) + \langle \nabla f(y),\,x-y\rangle - \tfrac{L}{2}\,\delta_x^2.
\end{equation}
Next, by Cauchy--Schwarz and the definition of
$S:=\sup_{z\in U}\|s_\theta(z)\|<\infty$~\textup{(A3)},
\[
\langle \nabla f(y),\,x-y\rangle
\;\ge\;
-\|\nabla f(y)\|\,\|x-y\|
\;\ge\;
- S\,\delta_x.
\]
Substituting into~\eqref{eq:quadratic_lower_bound} yields
\[
\log p_\theta(x)
\;\ge\;
\log p_\theta(y) - S\,\delta_x - \tfrac{L}{2}\,\delta_x^2.
\]
Exponentiating both sides gives
\[
p_\theta(x)\;\ge\; p_\theta(y)\,\exp\!\Big(-S\,\delta_x
  -\tfrac{L}{2}\,\delta_x^2\Big).
\]
Finally, since $y\in\mathcal{S}$, we have
$p_\theta(y)\ge C_b>0$~\textup{(P2)}, hence
\[
p_\theta(x)
\;\ge\;
C_b\,\exp\!\Big(-S\,\delta_x-\tfrac{L}{2}\,\delta_x^2\Big)
> 0,\]
which proves the claimed bound.
\end{proof}

\section{Score difference correlates with Hallucinations} \label{sec:1d-2d}
% in your preamble:

% then, where you want the figure:
\begin{figure}[htbp]
  \centering
  \includegraphics[width=\linewidth]{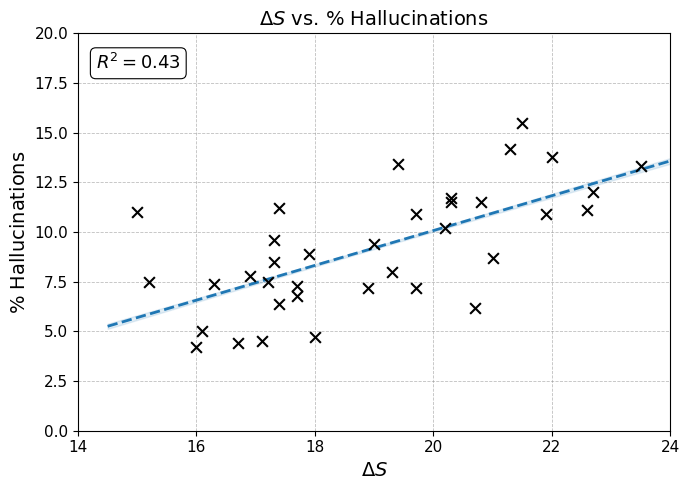}
  \caption{Increase in Score difference $\Delta s$ positively correlates with Hallucinations on Hands dataset.}
  \label{fig:score_hal}
\end{figure}
\noindent\textbf{Estimating} $\mathbf{s_{GT}}$:
For 1D and 2D datasets, we have closed form PDFs with fixed parametrs. Therefore, ground truth score can be obtained from closed form PDF: \(S_{GT}(x_t)=\frac{\sum_{m=1}^M-\frac{x_t-\mu_m}{\sigma^2}\exp\bigl(-\frac{(x_t-\mu_m)^2}{2\sigma^2}\bigr)}{\sum_{m=1}^M\exp\bigl(-\frac{(x_t-\mu_m)^2}{2\sigma^2}\bigr)}\). 
For image datasets, we do not have access to the groundtruth posterior induced by the forward noising process process at the inference time. Therefore, instead we invert the image to get $x_T$ from $x_0$ by forward noising give GT noise added, which is used for the calculating the expectation in equation 5.

\noindent\textbf{Results}:
We calculate $\Delta S$ as described in the section 4.1.
For 1D and 2D we already report the results in the Table 2 and describe them in the main paper. We also observe that the number of hallucinations is directly proportional to the score difference $\Delta S$ in the Hands dataset, as demonstrated in Fig.~\ref{fig:score_hal}.  This also motivates us to manipulate the learned score function to address the Hallucinations directly.

\section{Details on implementation of $\mathcal{L}_{VSM}$}

As seen in the main paper, \(\mathcal{L}_{VSM}\) penalizes small Jacobians of the learned score function. For data in \(\mathbb{R}^D\), the Jacobian of the score \(s : \mathbb{R}^D \to \mathbb{R}^D\) is a \(D\times D\) matrix.  In practice, we replace the exact derivatives with a centered finite‐difference approximation.

\noindent\textbf{1D case.}
When \(D=1\), \(S : \mathbb{R}\to\mathbb{R}\), the Jacobian reduces to the scalar derivative
\[
J_S(x)\;=\;\frac{d}{dx}S(x)
\;\approx\;
\frac{S(x + \varepsilon)\;-\;S(x - \varepsilon)}{2\,\varepsilon}\,.
\]

\noindent\textbf{2D case.}
When \(D=2\), write \(x=(x_1,x_2)\in\mathbb{R}^2\) and \(S=(S_1,S_2)\).  The Jacobian matrix \(J_S(x)\) has entries
\[
[J_S(x)]_{ij}
\;=\;
\frac{\partial S_i(x)}{\partial x_j}
\;\approx\;
\frac{S_i\bigl(x + \varepsilon\,e_j\bigr)\;-\;
      S_i\bigl(x - \varepsilon\,e_j\bigr)}{2\,\varepsilon}
\quad
(i,j=1,2),
\]
where \(e_1=(1,0)\), \(e_2=(0,1)\).  Equivalently,
\begin{align}
J_S(x)
\;\approx\;
\frac{1}{2\varepsilon}\,
&\begin{pmatrix}
S_1(x_1{+}\varepsilon,x_2) - S_1(x_1{-}\varepsilon,x_2)
& 
S_1(x_1,x_2{+}\varepsilon) - S_1(x_1,x_2{-}\varepsilon)
\\[6pt]
S_2(x_1{+}\varepsilon,x_2) - S_2(x_1{-}\varepsilon,x_2)
&
S_2(x_1,x_2{+}\varepsilon) - S_2(x_1,x_2{-}\varepsilon)
\end{pmatrix}.\nonumber
\end{align}

\noindent\textbf{Images case.} 
 The Jacobian of the Score is also the Precision matrix $-\Sigma^{-1}$. However, there are two problems in calculating the covariance matrix $\Sigma$, 1. Closed-form PDF is not available 2. Calculating and storing the Jacobian is not computationally feasible for high-dimensional image settings. Therefore, instead, we use the $\Sigma_{\theta}$ learned in the denoising process. We adopt the I-DDPM parameterization \cite{nichol2021improved} to learn variance, more details in section 4.3 of the main paper.  

\section{More details on the ChessImages dataset} \label{sec:chess_more}

\paragraph{Invalid Chessboard Detection:}

Section 5.1 in the main paper describes details about creating the ChessImages dataset. The validation module ensures that every generated chessboard image corresponds to a legal board configuration. This is achieved through a hybrid pipeline comprising visual and rule-based checks, designed to detect hallucinations automatically-board states violating chess semantics or displaying visual inconsistencies. We begin by reconstructing the FEN string from each rendered image using a template-matching-based parser, achieving 100\% reconstruction accuracy on the training set. However, given chess's combinatorial nature, no tractable algorithm can verify the reachability of arbitrary board states through legal move sequences. Hence, we instead focus on verifying the structural validity of the board state using syntactic and semantic criteria.For the validation of the FEN, we use \lstinline!status()!%
from the \lstinline!python-chess! library\footnote{\url{https://python-chess.readthedocs.io/en/latest/}}.
We construct a dataset with strong structural priors and verifiable correctness by enforcing these constraints. This eliminates manual annotation during hallucination detection and enables reproducible and objective evaluations in downstream generative modeling tasks.

A generated chessboard image is considered invalid if it meets any of the following criteria: (1) the extracted FEN string from the image has a similarity score below 50\%, indicating poor or ambiguous visual parsing; or (2) the parsed FEN fails legality checks using the python-chess library, such as having multiple kings of the same color, exceeding eight pawns per player, overlapping or missing pieces, or violating fundamental chess rules.

Below we list the rules used by the chess library’s \texttt{status()} check and, for each rule, we also show the images flagged as “hallucinated” because of violations of these rules in Fig.~\ref{fig:Chess_fig2} and Fig.~\ref{fig:Chess_fig2}.

(i) \textbf{Non‐empty board.}  
A valid FEN must contain at least one piece. Violations are flagged by \texttt{STATUS\_EMPTY}.

\medskip

(ii) \textbf{Exactly one white king.}  
There must be one (and only one) white king on the board. Violations are flagged by \\ \texttt{STATUS\_NO\_WHITE\_KING}, \texttt{STATUS\_TOO\_MANY\_KINGS}.

\medskip

(iii) \textbf{Exactly one black king.}  
There must be one (and only one) black king. Violations are flagged by \texttt{STATUS\_NO\_BLACK\_KING}, \texttt{STATUS\_TOO\_MANY\_KINGS}.

\medskip

(iv) \textbf{Piece-count limits.}  
No side may have more than 16 pieces. Violations are flagged by \texttt{STATUS\_TOO\_MANY\_WHITE\_PIECES}, \\ \texttt{STATUS\_TOO\_MANY\_BLACK\_PIECES}.

\medskip

(v) \textbf{Pawn-count limits.}  
No side may have more than eight pawns. Violations are flagged by \texttt{STATUS\_TOO\_MANY\_WHITE\_PAWNS}, \texttt{STATUS\_TOO\_MANY\_BLACK\_PAWNS}.

\medskip

(vi) \textbf{No pawns on back-rank.}  
Pawns may not appear on ranks 1 or 8. Violations are flagged by \texttt{STATUS\_PAWNS\_ON\_BACKRANK}.

\medskip

(vii) \textbf{Legal castling rights.}  
Castling flags must match the king/rook placement. Violations are flagged by \texttt{STATUS\_BAD\_CASTLING\_RIGHTS}.

\medskip

(viii) \textbf{Valid en passant.}  
The ep‐target square must be reachable by a two‐square pawn move. Violations are flagged by \\ \texttt{STATUS\_INVALID\_EP\_SQUARE}.

\medskip

(ix) \textbf{No opposite‐side check.}  
The side that is not to move cannot be checked. Violations are flagged by \texttt{STATUS\_OPPOSITE\_CHECK}.

\medskip

(x) \textbf{Max two checking pieces.}  
At most two pieces may deliver a check. Violations are flagged by \texttt{STATUS\_TOO\_MANY\_CHECKERS}.

\medskip

(xi) \textbf{Possible check sequence.}  
Checks must arise via a legal move (including ep pushes). Violations are flagged by \\  \texttt{STATUS\_IMPOSSIBLE\_CHECK}.

A Standard FEN string contains \texttt{"<PiecePlacement> \\<ActiveColor> <CastlingRights> <EnPassant> \\<HalfmoveClock> <FullmoveNumber>"}. Template matching can only give us \texttt{<PiecePlacement>}. Therefore, the rules (vii, viii, ix) that use the information from parts of the FEN other than \\  \texttt{<PiecePlacement>} are ignored in our work. 

\begin{figure}[htbp]
  \centering
  \includegraphics[width=0.95\linewidth]{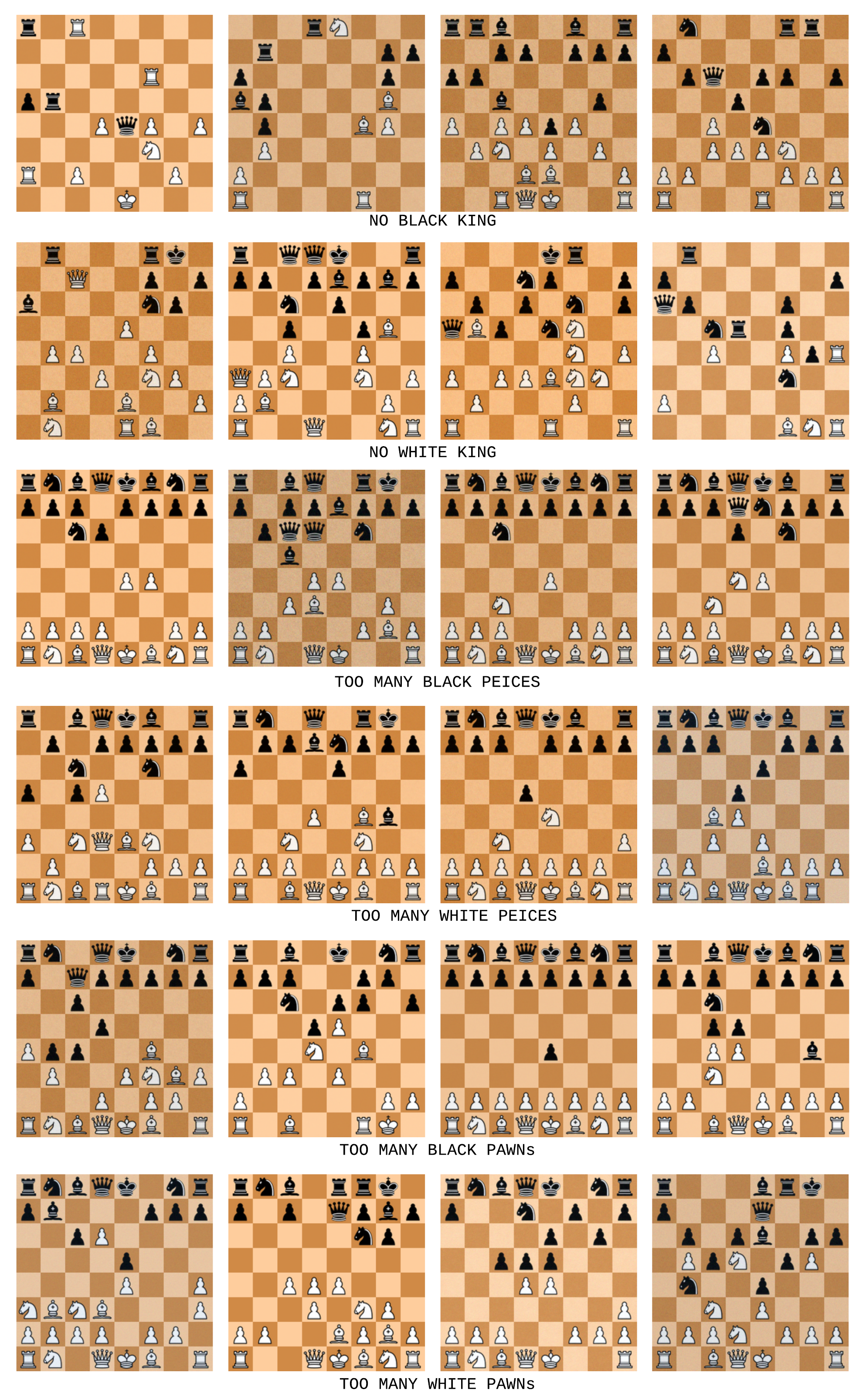}
  \includegraphics[width=0.95\linewidth]{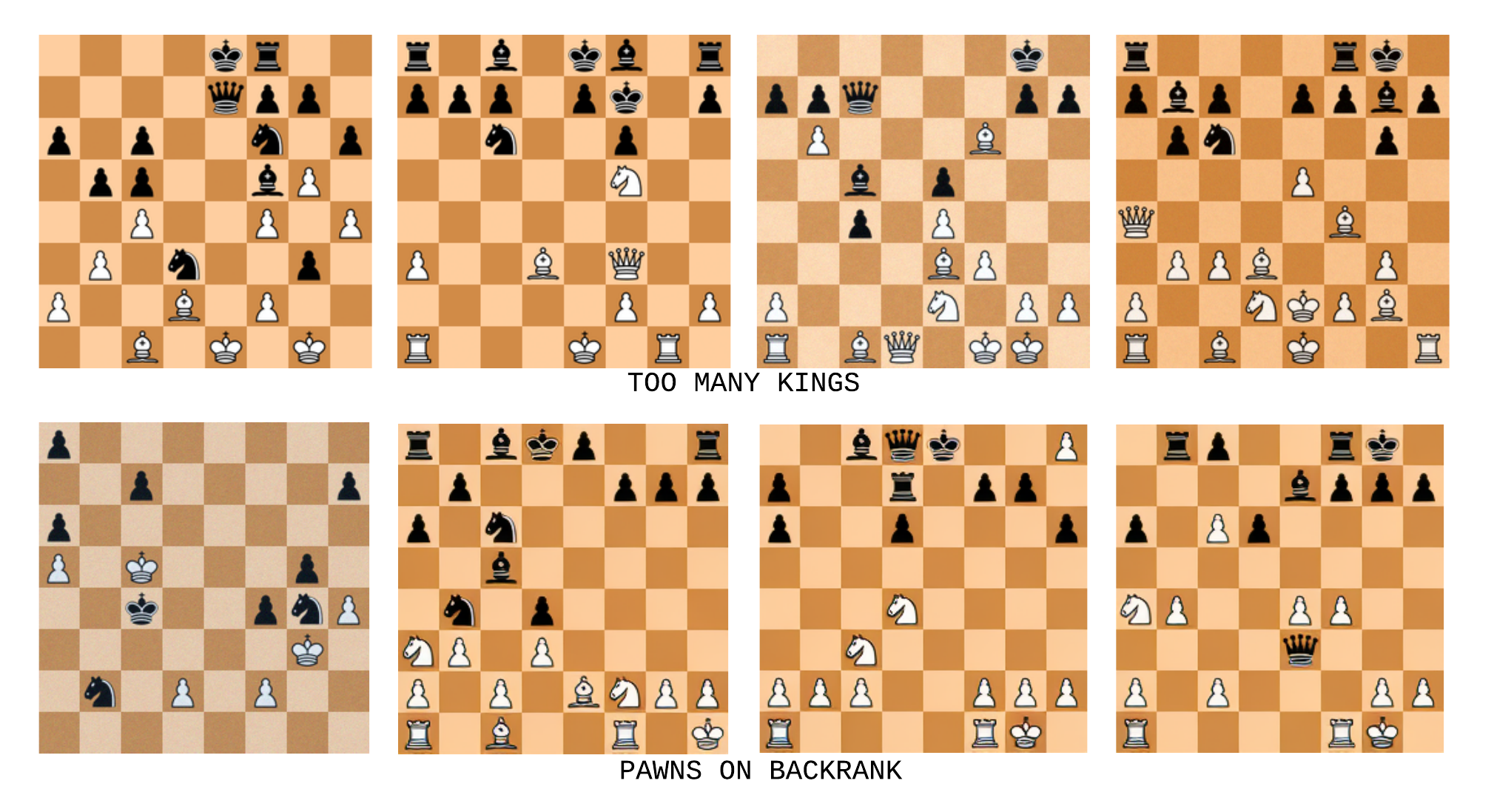}
  \caption{Generated images marked Hallucinated for the reasons mentioned at the bottom of each row.}
  \label{fig:Chess_fig2}
\end{figure}

Demonstrated in Tab. \ref{tab:novel_boards} we compare the total number of valid novel boards generated by all the methods. Proposed methods can be utilized as more robust data augmentation technics with high rule prior datasets such as proposed ChessImages dataset.

\begin{minipage}{0.6\linewidth}
  \centering
  \scalebox{0.95}{%
    \begin{tabular}{@{}l c@{}}
      \toprule
      \textbf{Method} & \textbf{\# Novel Boards} \\
      \midrule
      DDPM \cite{ho2020denoising}                    & 60842 \\
      % LDM–UC \cite{ldm} &  \\
      Variance Learning \cite{nichol2021improved} & 69950 \\
      $\mathcal{L}_{VSM}$ &  77421 \\
      \bottomrule
    \end{tabular}%
  }
\end{minipage}%
\hfill
\begin{minipage}{0.3 \linewidth}
\captionof{table}{Number of valid boards that are novel out of 190K generated samples. Both Variance learning and $\mathcal{L}_{VSM}$ gnerate considerably large number of valid novel boards as compared to baselines.}
\label{tab:novel_boards}
\end{minipage}

\section{Effect of Dataset Size}
We investigate how the size of the training set influences hallucination rates. To ensure consistent comparisons, we construct three nested subsets containing 75\%, 50\%, and 25\% of the full dataset—each smaller subset being wholly contained within the next larger one. As shown in Tab.~\ref{tab:dataset_size}, shrinking the training set reduces the support from diverse examples, which in turn increases the incidence of hallucinations. This underscores the crucial role of ample data support in achieving reliable image generation.

\noindent
\begin{minipage}[t]{0.60\linewidth}
  \centering
  \scalebox{0.95}{%
    \begin{tabular}{@{} c  c  c  c @{}}
      \toprule
      \textbf{Dataset Size} 
        & \multicolumn{2}{c}{\textbf{Shapes}} 
        & \textbf{ChessImages} \\
      \cmidrule(lr){2-3} \cmidrule(lr){4-4}
      & \textbf{DDPM} 
      & $\mathcal{L}_{VSM}$ 
      & $\mathcal{L}_{VSM}$ \\
      \midrule
      25  & 89.74 & 20.50 & 80.33 \\
      50  & 57.16 & 13.16 & 65.35 \\
      75  & 55.16 &  5.66 & 61.75 \\
      100 & 29.50 &  3.00 & 55.0 \\
      \bottomrule
    \end{tabular}%
  }
\end{minipage}\hfill
\begin{minipage}{0.28\linewidth}
  \captionof{table}{Effect of training-set size on hallucination rates (\%): for the Shapes dataset we compare DDPM vs.\ $\mathcal{L}_{VSM}$, and for ChessImages we report on $\mathcal{L}_{VSM}$.}
  \label{tab:dataset_size}
\end{minipage}

\section{Effect of number of Denoising Steps on Hallucinations}
We investigate how varying the number of denoising steps during inference affects hallucination rate on the Chess dataset. As~\cref{tab:rho_all} shows, there is no discernible relationship between step count and hallucination rate. Although fewer denoising steps are known to degrade overall image fidelity, they do not consistently alter the number of hallucinations.
\begin{table}[t]
\centering
\begin{tabular}{@{}ccc@{}}
\toprule
\multirow{2}{*}{Denoising Steps} & \multicolumn{2}{c}{Hallucinations (\%)} \\
\cmidrule(lr){2-2}\cmidrule(lr){3-3}
 & DDPM & $\mathcal{L}_{\mathrm{VSM}}$ \\
\midrule
50  & 61.75 & 57.75 \\
100 & 62.00 & 53.00 \\
150 & 69.50 & 55.75 \\
200 & 64.00 & 51.00 \\
250 & 66.25 & 55.00 \\
\bottomrule
\end{tabular}
\caption{Effect of denoising steps on hallucinations (\%) on the Chess dataset.}
\label{tab:rho_all}
\end{table}

\section{LDM Prompt Tuning ~\cite{mahajan2024prompting}} \label{sec:prompt-tuning}
For the conditional LDM (LDM-C) setting, we condition generation on text prompts: a single default prompt for the Hands dataset, and class-embedded prompts for MNIST. In the prompt-tuning (LDM-PT) setting, we further fine-tune these prompts to mitigate the hallucinations we observed (see Table 3). For each dataset, we crafted 20 distinct prompts and, at inference time, randomly select one to drive image synthesis. We observe that this prompt-tuning strategy substantially reduces hallucination rates on both Hands and MNIST.
% \newpage

\textbf{MNIST:}
Default Prompt:
\begin{lstlisting}
    ["Image of handwritten digit <digit_class>"]
\end{lstlisting}
% \newpage
Finetuned Prompts:
\begin{lstlisting}
    [ # I. Zero
    "MNIST-style handwritten 'zero': thin white strokes, centered on a clean black background, no extra marks.",
    "MNIST-style handwritten 'zero': minimal white loop, centered on black, uniform thickness, no noise.",

    # II. One
    "MNIST-style handwritten 'one': single thin white vertical stroke, centered on black, no stray pixels.",
    "MNIST-style handwritten 'one': clean white digit one, straight line, centered on black, isolated.",

    # III. Two
    "MNIST-style handwritten 'two': crisp white strokes, centered on black, no overlapping or smudges.",
    "MNIST-style handwritten 'two': clear white digit two, centered on black, uniform lines, no noise.",

    # IV. Three
    "MNIST-style handwritten 'three': two smooth thin white strokes, centered on black, no extra artifacts.",
    "MNIST-style handwritten 'three': neat white digit three, centered on black, distinct curves, clean.",

    # V. Four
    "MNIST-style handwritten 'four': intersecting thin white strokes, centered on black, no stray marks.",
    "MNIST-style handwritten 'four': crisp white digit four, centered on black, clear junctions.",

    # VI. Five
    "MNIST-style handwritten 'five': clear thin white strokes, centered on black, no overlapping lines.",
    "MNIST-style handwritten 'five': sharp white digit five, centered on black, isolated strokes.",

    # VII. Six
    "MNIST-style handwritten 'six': continuous thin white stroke, centered on black, no breaks.",
    "MNIST-style handwritten 'six': clean white digit six, rounded form, centered on black, no noise.",

    # VIII. Seven
    "MNIST-style handwritten 'seven': two thin white strokes, centered on black, no extra marks.",
    "MNIST-style handwritten 'seven': neat white digit seven, centered on black, uniform thickness.",

    # IX. Eight
    "MNIST-style handwritten 'eight': two distinct thin white loops, centered on black, no distortions.",
    "MNIST-style handwritten 'eight': symmetric white digit eight, centered on black, clear separation.",

    # X. Nine
    "MNIST-style handwritten 'nine': thin white strokes, centered on black, isolated and clean.",
    "MNIST-style handwritten 'nine': crisp white digit nine, centered on black, no extra pixels."]
\end{lstlisting}
% \newpage

\textbf{Hands:}

Default Prompt:
\begin{lstlisting}
    ["Close-up high quality image of a human hand on White background"]
\end{lstlisting}
Finetuned Prompts:
\begin{lstlisting}
    ["High-resolution photo of a human hand, palm fully open with five fingers (thumb, index, middle, ring, pinky) spread naturally, plain white background.",
    
    "Close-up shot of an open human palm showing all five fingers in correct thumb-to-pinky order, flat facing the camera, on white.",
    
    "Photograph of a human hand with palm wide open, five straight fingers (thumb - index - middle - ring - little finger), against a white backdrop.",
    
    "Studio image of an open palm displaying five fingers in proper sequence-thumb at left, pinky at right-on a clean white background.",
    
    "Realistic photo of a single human palm, five fingers fully extended in thumb-to-pinky order, flat and facing forward, white background.",
    
    "High-quality image of a human hand, palm completely open, five fingers aligned anatomically (thumb, index, middle, ring, pinky), white backdrop.",
    
    "Close-up of an open palm with five straight fingers, thumb on the left and pinky on the right, on solid white.",
    
    "Photorealistic shot of a fully opened palm showing five fingers in correct order, flat against a white background.",
    
    "Sharp photo of a human hand, palm fully extended with thumb, index, middle, ring, and little finger visible in order, white background.",
    
    "Clean studio portrait of an open palm-five fingers (thumb through pinky) splayed evenly-on a white backdrop.",
    
    "High-resolution image of an open palm with five anatomically ordered fingers, thumb first then index, middle, ring, and pinky, against white.",
    
    "Close-up studio photo of a human palm fully open, showing five straight fingers in thumb-to-pinky sequence, white background.",
    
    "Real-life shot of an open hand with palm facing camera, five fingers (thumb - index - middle - ring - little) in order, white backdrop.",
    
    "Crisp image of an open palm with five fingers aligned anatomically, thumb on the left edge, pinky on the right, plain white background.",
    
    "Photograph of a human palm flat and facing forward, five fingers visible in correct anatomical order, white background.",
    
    "Studio-style image of an open hand-five fingers from thumb to pinky-fully extended and flat against white.",
    
    "Close-up of a human palm with five distinct fingers, starting from thumb then index, middle, ring, little, on a white backdrop.",
    
    "Detailed photo of an open palm showing five fingers in sequence, thumb at outer edge, pinky at other, on solid white.",
    
    "High-detail shot of a human palm fully opened, five straight fingers in anatomical order, flat and white background.",
    
    "Clear photo of a human hand, palm fully open with thumb, index, middle, ring, and pinky fingers visible in order on a white background."]
\end{lstlisting}

\section{More details on the Cards dataset}
In \cref{fig:cards_fig} we show more samples of the images that are hallucinated by the rules mentioned in the main paper. 

% (i) The number of symbols does not match the card value. 

% (ii) The color of the symbol is incorrect.

% (iii) Invalid or Missing symbols

% (iv) Multiple conflicting symbols in a single card.

\begin{figure}[htbp]
  \centering
  \includegraphics[width=0.95\linewidth]{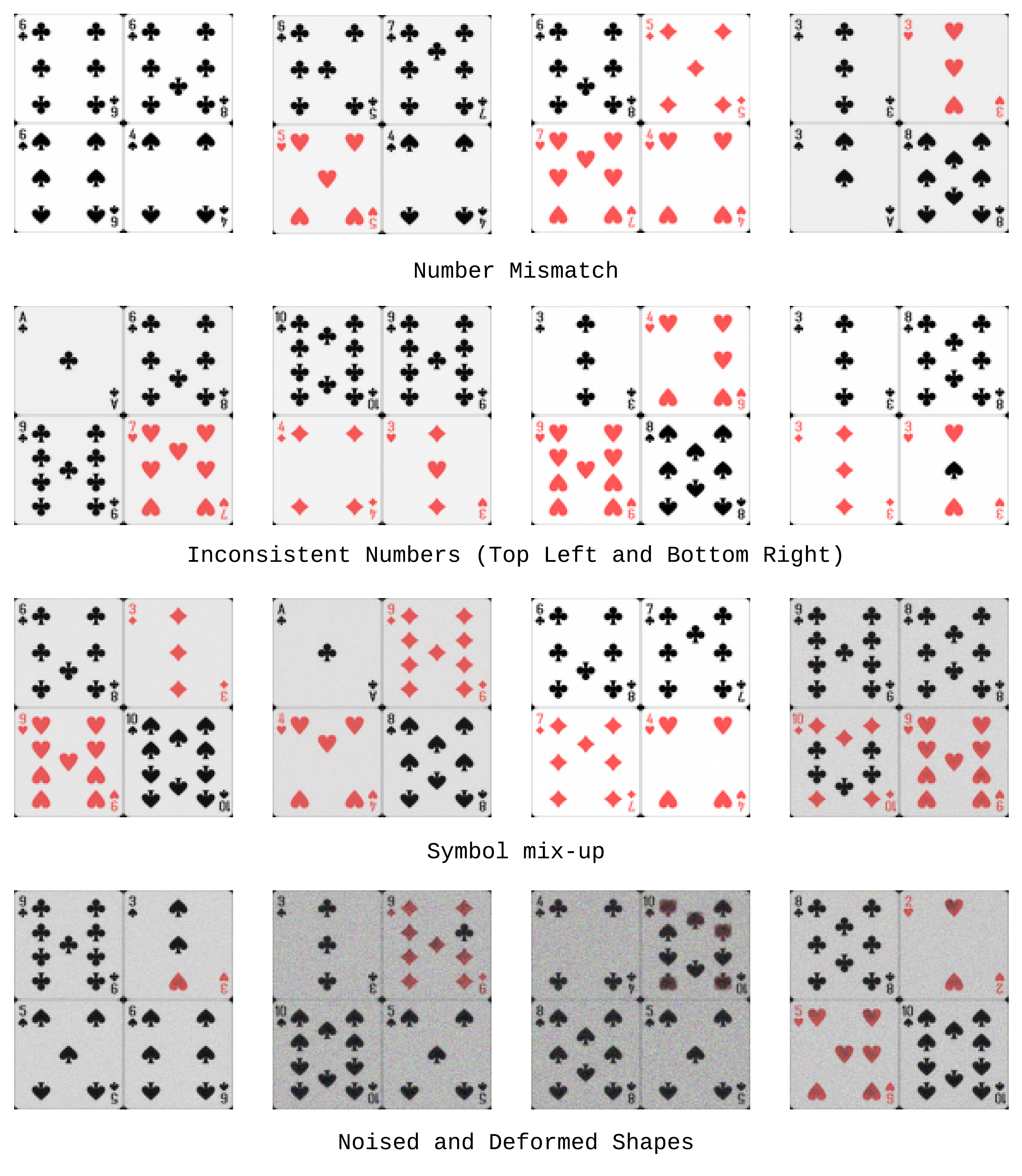}
  \caption{Generated images marked Hallucinated for the reasons mentioned at the bottom of each row. }
  \label{fig:cards_fig}
\end{figure}

\section{Implementation Details}
For 1D and 2D datasets, our code is built upon \cite{mode_interpolation}. For Image datasets with variance learning and $\mathcal{L}_{VSM}$ implementation, we build upon \cite{nichol2021improved}. All experiments are carried out on 8 Nvidia A6000 GPUs. All the quantitative results on the Image datasets are obtained using six seeds and generating 100 images per seed. We also used six seeds for the 1D and 2D cases, generated 1 million sample points per seed, and reported the average. For the LDM baseline, we use the codebase provided by \cite{ldm}. Specifically, for LDM-C, we initialized our diffusion model from the Stable Diffusion checkpoint pretrained on ImageNet and used the CLIP text encoder to extract text embeddings. For unconditional training, we train LDM from scratch. For \cite{adapt_attention}, we directly use the quantitative results reported in the original paper. 

\begin{figure}[t]
  \centering
    \includegraphics[width=\linewidth]{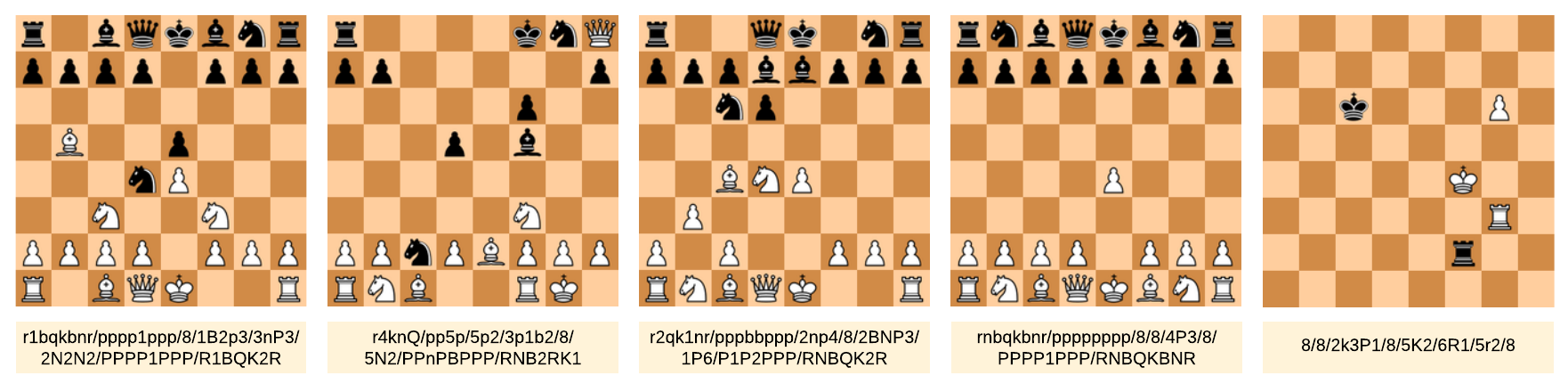}
  \caption{Example samples from the proposed \textit{ChessImages} dataset. 
\textit{Top}: a generated chessboard configuration. 
\textit{Bottom}: its corresponding Forsyth–Edwards Notation (FEN) string, 
providing an exact symbolic representation of the board state.}
  \label{fig:example_chess}
\end{figure}

\section{Additonal qualitative samples on ImageNet-1K}

We provide additional qualitative comparisons on the ImageNet-1K dataset in Figure \cref{fig:cats-dogs}. We use the LDM model trained without $\mathcal{L}_{VSM}$ as the baseline (shown in red) and compare it against our method trained with $\mathcal{L}_{VSM}$ (shown in green). The baseline frequently produces deformed objects and incompletely denoised samples, resulting in images that deviate from the training data distribution. In contrast, our method mitigates these failure cases and generates samples that are more coherent, well-formed, and closely aligned with the training data distribution. Quantitative results are provided in Table 3 of the main paper.

\begin{figure}[h]
  \centering
    \includegraphics[width=\linewidth]{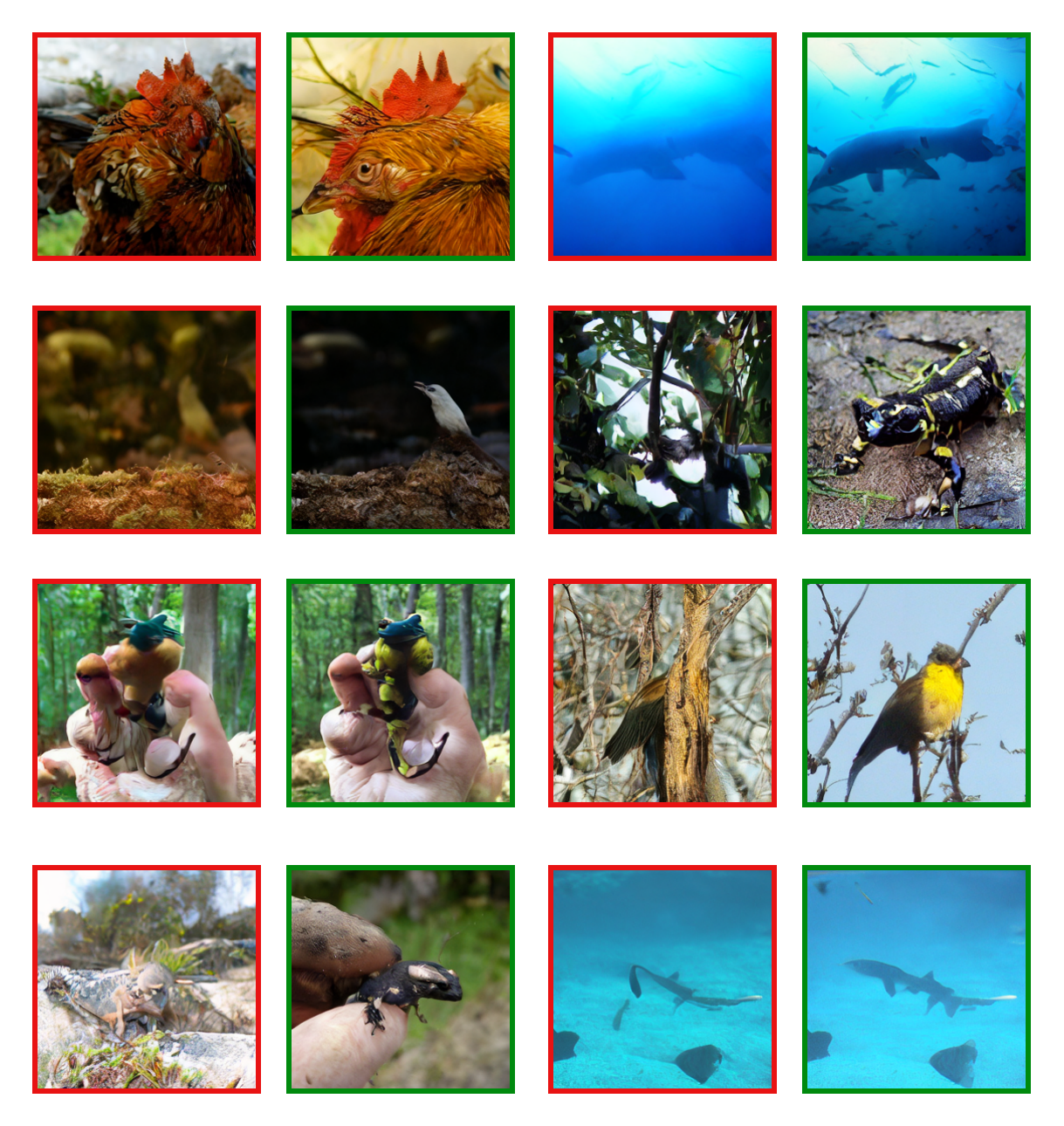}
  \caption{We observe that our method corrects the deformed objects, incompletely denoised images on the ImageNet-1K dataset.}
  \label{fig:cats-dogs}
\end{figure}

\bibliographystyle{ACM-Reference-Format}
\bibliography{sample-base}

@String{Computing = "Computing" }

@String{Computer = "{IEEE} Computer" }

@String{Academic = "Academic Press" }

@misc{umh,
      title={Verifying the Union of Manifolds Hypothesis for Image Data}, 
      author={Bradley C. A. Brown and Anthony L. Caterini and Brendan Leigh Ross and Jesse C. Cresswell and Gabriel Loaiza-Ganem},
      year={2023},
      eprint={2207.02862},
      archivePrefix={arXiv},
      primaryClass={stat.ML},
      url={https://arxiv.org/abs/2207.02862}, 
}

@article{ho2020denoising,
  title={Denoising diffusion probabilistic models},
  author={Ho, Jonathan and Jain, Ajay and Abbeel, Pieter},
  journal={Advances in neural information processing systems},
  volume={33},
  pages={6840--6851},
  year={2020}
}

@article{guo2024diffusion,
  title={Diffusion models in bioinformatics and computational biology},
  author={Guo, Zhiye and Liu, Jian and Wang, Yanli and Chen, Mengrui and Wang, Duolin and Xu, Dong and Cheng, Jianlin},
  journal={Nature reviews bioengineering},
  volume={2},
  number={2},
  pages={136--154},
  year={2024},
  publisher={Nature Publishing Group UK London}
}

@online{stability_sd35_2024,
  author       = {{Stability AI}},
  title        = {Introducing Stable Diffusion 3.5},
  year         = {2024},
  month        = {10},
  url          = {https://stability.ai/news/introducing-stable-diffusion-3-5},
  note         = {Official announcement of the SD~3.5 model family}
}

@article{hao2023safety,
  title={Safety and fairness for content moderation in generative models},
  author={Hao, Susan and Kumar, Piyush and Laszlo, Sarah and Poddar, Shivani and Radharapu, Bhaktipriya and Shelby, Renee},
  journal={arXiv preprint arXiv:2306.06135},
  year={2023}
}

@article{wewer2025spatial,
  title={Spatial reasoning with denoising models},
  author={Wewer, Christopher and Pogodzinski, Bart and Schiele, Bernt and Lenssen, Jan Eric},
  journal={arXiv preprint arXiv:2502.21075},
  year={2025}
}

@inproceedings{lu2025towards,
  title={Towards understanding text hallucination of diffusion models via local generation bias},
  author={Lu, Rui and Wang, Runzhe and Lyu, Kaifeng and Jiang, Xitai and Huang, Gao and Wang, Mengdi},
  booktitle={The Thirteenth International Conference on Learning Representations},
  year={2025}
}

@misc{kim2024tacklingstructuralhallucinationimage,
      title={Tackling Structural Hallucination in Image Translation with Local Diffusion}, 
      author={Seunghoi Kim and Chen Jin and Tom Diethe and Matteo Figini and Henry F. J. Tregidgo and Asher Mullokandov and Philip Teare and Daniel C. Alexander},
      year={2024},
      eprint={2404.05980},
      archivePrefix={arXiv},
      primaryClass={cs.CV},
      url={https://arxiv.org/abs/2404.05980}, 
}

@misc{trust_gen,
      title={On the Trustworthiness of Generative Foundation Models: Guideline, Assessment, and Perspective}, 
      author={Yue Huang and Chujie Gao and Siyuan Wu and Haoran Wang and Xiangqi Wang and Yujun Zhou and Yanbo Wang and Jiayi Ye and Jiawen Shi and Qihui Zhang and Yuan Li and Han Bao and Zhaoyi Liu and Tianrui Guan and Dongping Chen and Ruoxi Chen and Kehan Guo and Andy Zou and Bryan Hooi Kuen-Yew and Caiming Xiong and Elias Stengel-Eskin and Hongyang Zhang and Hongzhi Yin and Huan Zhang and Huaxiu Yao and Jaehong Yoon and Jieyu Zhang and Kai Shu and Kaijie Zhu and Ranjay Krishna and Swabha Swayamdipta and Taiwei Shi and Weijia Shi and Xiang Li and Yiwei Li and Yuexing Hao and Zhihao Jia and Zhize Li and Xiuying Chen and Zhengzhong Tu and Xiyang Hu and Tianyi Zhou and Jieyu Zhao and Lichao Sun and Furong Huang and Or Cohen Sasson and Prasanna Sattigeri and Anka Reuel and Max Lamparth and Yue Zhao and Nouha Dziri and Yu Su and Huan Sun and Heng Ji and Chaowei Xiao and Mohit Bansal and Nitesh V. Chawla and Jian Pei and Jianfeng Gao and Michael Backes and Philip S. Yu and Neil Zhenqiang Gong and Pin-Yu Chen and Bo Li and Dawn Song and Xiangliang Zhang},
      year={2025},
      eprint={2502.14296},
      archivePrefix={arXiv},
      primaryClass={cs.CY},
      url={https://arxiv.org/abs/2502.14296}, 
}

@online{stanford_ai_index_2025,
  author       = {{Stanford HAI}},
  title        = {The 2025 AI Index Report},
  year         = {2025},
  url          = {https://hai.stanford.edu/ai-index/2025-ai-index-report},
  note         = {Reports 78\% of organizations using AI in 2024}
}

@online{adobe_firefly_22b_2025,
  author       = {{Adobe}},
  title        = {Adobe Firefly: The next evolution of creative AI is here},
  year         = {2025},
  month        = {4},
  url          = {https://blog.adobe.com/en/publish/2025/04/24/adobe-firefly-next-evolution-creative-ai-is-here},
  note         = {Adobe reports 22B+ Firefly-generated assets worldwide}
}

@inproceedings{mode_interpolation,
 author = {Aithal, Sumukh K and Maini, Pratyush and Lipton, Zachary C. and Kolter, J. Zico},
 booktitle = {Advances in Neural Information Processing Systems},
 editor = {A. Globerson and L. Mackey and D. Belgrave and A. Fan and U. Paquet and J. Tomczak and C. Zhang},
 pages = {134614--134644},
 publisher = {Curran Associates, Inc.},
 title = {Understanding Hallucinations in Diffusion Models through Mode Interpolation},
 url = {https://proceedings.neurips.cc/paper_files/paper/2024/file/f29369d192b13184b65c6d2515474d78-Paper-Conference.pdf},
 volume = {37},
 year = {2024}
}

@article{song2019generative,
  title={Generative modeling by estimating gradients of the data distribution},
  author={Song, Yang and Ermon, Stefano},
  journal={Advances in neural information processing systems},
  volume={32},
  year={2019}
}

@article{interpolation,
  title={On the Interpolation Effect of Score Smoothing},
  author={Chen, Zhengdao},
  journal={arXiv preprint arXiv:2502.19499},
  year={2025}
}

@article{adapt_attention,
  title={Mitigating Hallucinations in Diffusion Models through Adaptive Attention Modulation},
  author={Oorloff, Trevine and Yacoob, Yaser and Shrivastava, Abhinav},
  journal={arXiv preprint arXiv:2502.16872},
  year={2025}
}

@inproceedings{nichol2021improved,
  title={Improved denoising diffusion probabilistic models},
  author={Nichol, Alexander Quinn and Dhariwal, Prafulla},
  booktitle={International conference on machine learning},
  pages={8162--8171},
  year={2021},
  organization={PMLR}
}

@article{ddim,
  title={Denoising diffusion implicit models},
  author={Song, Jiaming and Meng, Chenlin and Ermon, Stefano},
  journal={arXiv preprint arXiv:2010.02502},
  year={2020}
}

@inproceedings{ldm,
  title={High-resolution image synthesis with latent diffusion models},
  author={Rombach, Robin and Blattmann, Andreas and Lorenz, Dominik and Esser, Patrick and Ommer, Bj{\"o}rn},
  booktitle={Proceedings of the IEEE/CVF conference on computer vision and pattern recognition},
  pages={10684--10695},
  year={2022}
}

@misc{russakovsky2015imagenetlargescalevisual,
      title={ImageNet Large Scale Visual Recognition Challenge}, 
      author={Olga Russakovsky and Jia Deng and Hao Su and Jonathan Krause and Sanjeev Satheesh and Sean Ma and Zhiheng Huang and Andrej Karpathy and Aditya Khosla and Michael Bernstein and Alexander C. Berg and Li Fei-Fei},
      year={2015},
      eprint={1409.0575},
      archivePrefix={arXiv},
      primaryClass={cs.CV},
      url={https://arxiv.org/abs/1409.0575}, 
}

@misc{improvedprecision,
      title={Improved Precision and Recall Metric for Assessing Generative Models}, 
      author={Tuomas Kynkäänniemi and Tero Karras and Samuli Laine and Jaakko Lehtinen and Timo Aila},
      year={2019},
      eprint={1904.06991},
      archivePrefix={arXiv},
      primaryClass={stat.ML},
      url={https://arxiv.org/abs/1904.06991}, 
}

@inproceedings{mahajan2024prompting,
  title={Prompting hard or hardly prompting: Prompt inversion for text-to-image diffusion models},
  author={Mahajan, Shweta and Rahman, Tanzila and Yi, Kwang Moo and Sigal, Leonid},
  booktitle={Proceedings of the IEEE/CVF Conference on Computer Vision and Pattern Recognition},
  pages={6808--6817},
  year={2024}
}

@article{song2020score,
  title={Score-based generative modeling through stochastic differential equations},
  author={Song, Yang and Sohl-Dickstein, Jascha and Kingma, Diederik P and Kumar, Abhishek and Ermon, Stefano and Poole, Ben},
  journal={arXiv preprint arXiv:2011.13456},
  year={2020}
}

@misc{triaridis2025mitigatingdiffusionmodelhallucinations,
      title={Mitigating Diffusion Model Hallucinations with Dynamic Guidance}, 
      author={Kostas Triaridis and Alexandros Graikos and Aggelina Chatziagapi and Grigorios G. Chrysos and Dimitris Samaras},
      year={2025},
      eprint={2510.05356},
      archivePrefix={arXiv},
      primaryClass={cs.CV},
      url={https://arxiv.org/abs/2510.05356}, 
}

@article{saha2023valued,
  title={VALUED--Vision and Logical Understanding Evaluation Dataset},
  author={Saha, Soumadeep and Saha, Saptarshi and Garain, Utpal},
  journal={arXiv preprint arXiv:2311.12610},
  year={2023}
}

@article{wu2023ar,
  title={Ar-diffusion: Auto-regressive diffusion model for text generation},
  author={Wu, Tong and Fan, Zhihao and Liu, Xiao and Zheng, Hai-Tao and Gong, Yeyun and Jiao, Jian and Li, Juntao and Guo, Jian and Duan, Nan and Chen, Weizhu and others},
  journal={Advances in Neural Information Processing Systems},
  volume={36},
  pages={39957--39974},
  year={2023}
}

@article{shen2023finetuning,
  title={Finetuning text-to-image diffusion models for fairness},
  author={Shen, Xudong and Du, Chao and Pang, Tianyu and Lin, Min and Wong, Yongkang and Kankanhalli, Mohan},
  journal={arXiv preprint arXiv:2311.07604},
  year={2023}
}

@InProceedings{bhosale2025pathdiff,
    author    = {Bhosale, Mahesh and Wasi, Abdul and Zhai, Yuanhao and Tian, Yunjie and Border, Samuel and Xi, Nan and Sarder, Pinaki and Yuan, Junsong and Doermann, David and Gong, Xuan},
    title     = {PathDiff: Histopathology Image Synthesis with Unpaired Text and Mask Conditions},
    booktitle = {Proceedings of the IEEE/CVF International Conference on Computer Vision (ICCV)},
    month     = {October},
    year      = {2025},
    pages     = {22415-22424}
}

@inproceedings{saharia2022palette,
  title={Palette: Image-to-image diffusion models},
  author={Saharia, Chitwan and Chan, William and Chang, Huiwen and Lee, Chris and Ho, Jonathan and Salimans, Tim and Fleet, David and Norouzi, Mohammad},
  booktitle={ACM SIGGRAPH 2022 conference proceedings},
  pages={1--10},
  year={2022}
}

@article{li2022diffusion,
  title={Diffusion-lm improves controllable text generation},
  author={Li, Xiang and Thickstun, John and Gulrajani, Ishaan and Liang, Percy S and Hashimoto, Tatsunori B},
  journal={Advances in neural information processing systems},
  volume={35},
  pages={4328--4343},
  year={2022}
}

@article{hands11k,
author = {Afifi, Mahmoud},
title = {11K Hands: Gender recognition and biometric identification using a large dataset of hand images},
year = {2019},
issue_date = {Aug 2019},
publisher = {Kluwer Academic Publishers},
address = {USA},
volume = {78},
number = {15},
issn = {1380-7501},
url = {https://doi.org/10.1007/s11042-019-7424-8},
doi = {10.1007/s11042-019-7424-8},
journal = {Multimedia Tools Appl.},
month = aug,
pages = {20835–20854},
numpages = {20}
}

@inproceedings{kushwaha2025diff,
  title={Diff-SAGe: End-to-End Spatial Audio Generation Using Diffusion Models},
  author={Kushwaha, Saksham Singh and Ma, Jianbo and Thomas, Mark RP and Tian, Yapeng and Bruni, Avery},
  booktitle={ICASSP 2025-2025 IEEE International Conference on Acoustics, Speech and Signal Processing (ICASSP)},
  pages={1--5},
  year={2025},
  organization={IEEE}
}

@ARTICLE{mnist,
  author={Lecun, Y. and Bottou, L. and Bengio, Y. and Haffner, P.},
  journal={Proceedings of the IEEE}, 
  title={Gradient-based learning applied to document recognition}, 
  year={1998},
  volume={86},
  number={11},
  pages={2278-2324},
  doi={10.1109/5.726791}}

@article{fld,
  title={Feature likelihood divergence: evaluating the generalization of generative models using samples},
  author={Jiralerspong, Marco and Bose, Joey and Gemp, Ian and Qin, Chongli and Bachrach, Yoram and Gidel, Gauthier},
  journal={Advances in Neural Information Processing Systems},
  volume={36},
  pages={33095--33119},
  year={2023}
}

@inproceedings{alaa2022faithful,
  title={How faithful is your synthetic data? sample-level metrics for evaluating and auditing generative models},
  author={Alaa, Ahmed and Van Breugel, Boris and Saveliev, Evgeny S and Van Der Schaar, Mihaela},
  booktitle={International conference on machine learning},
  pages={290--306},
  year={2022},
  organization={PMLR}
}

@article{pham2025memorization,
  title   = {Memorization to Generalization: Emergence of Diffusion Models from Associative Memory},
  author  = {Pham, Bao and Raya, Gabriel and Negri, Matteo and Zaki, Mohammed J. and Ambrogioni, Luca and Krotov, Dmitry},
  journal = {arXiv preprint arXiv:2505.21777},
  year    = {2025}
}

@article{ho2020ddpm,
  title={Denoising Diffusion Probabilistic Models},
  author={Ho, Jonathan and Jain, Ajay and Abbeel, Pieter},
  journal={NeurIPS},
  year={2020},
  url={https://arxiv.org/abs/2006.11239}
}

@article{song2020sde,
  title={Score-Based Generative Modeling through Stochastic Differential Equations},
  author={Song, Yang and Sohl-Dickstein, Jascha and Kingma, Diederik P. and Kumar, Abhishek and Ermon, Stefano and Poole, Ben},
  journal={ICLR},
  year={2021},
  url={https://arxiv.org/abs/2011.13456}
}

@inproceedings{MengSongLiErmon2021,
  title        = {Estimating High Order Gradients of the Data Distribution by Denoising},
  author       = {Chenlin Meng and Yang Song and Wenzhe Li and Stefano Ermon},
  booktitle    = {NeurIPS},
  year         = {2021},
  pages        = {12477--12488},
}

@article{alger2024point,
  title={Point spread function approximation of high-rank Hessians with locally supported nonnegative integral kernels},
  author={Alger, Nick and Hartland, Tucker and Petra, Noemi and Ghattas, Omar},
  journal={SIAM Journal on Scientific Computing},
  volume={46},
  number={3},
  pages={A1658--A1689},
  year={2024},
  publisher={SIAM}
}

@book{folland1999real,
  title={Real Analysis: Modern Techniques and Their Applications},
  author={Folland, Gerald B.},
  year={1999},
  publisher={John Wiley \& Sons},
  edition={2nd},
  address={New York}
}

@misc{yang2024lipschitzsingularitiesdiffusionmodels,
      title={Lipschitz Singularities in Diffusion Models}, 
      author={Zhantao Yang and Ruili Feng and Han Zhang and Yujun Shen and Kai Zhu and Lianghua Huang and Yifei Zhang and Yu Liu and Deli Zhao and Jingren Zhou and Fan Cheng},
      year={2024},
      eprint={2306.11251},
      archivePrefix={arXiv},
      primaryClass={cs.CV},
      url={https://arxiv.org/abs/2306.11251}, 
}

@book{rudin1976principles,
  title={Principles of Mathematical Analysis},
  author={Rudin, Walter},
  edition={3rd},
  year={1976},
  publisher={McGraw-Hill},
  address={New York},
  note={Extreme Value Theorem: a continuous function on a compact set attains a minimum and maximum}
}

@misc{python-chess,
  title        = {python-chess: a chess library for Python},
  howpublished = {\url{https://python-chess.readthedocs.io/en/latest/}},
  note         = {Accessed: YYYY-MM-DD}
}

@misc{devulapally2025textencoderobjectlevelwatermarking,
      title={Your Text Encoder Can Be An Object-Level Watermarking Controller}, 
      author={Naresh Kumar Devulapally and Mingzhen Huang and Vishal Asnani and Shruti Agarwal and Siwei Lyu and Vishnu Suresh Lokhande},
      year={2025},
      eprint={2503.11945},
      archivePrefix={arXiv},
      primaryClass={cs.CV},
      url={https://arxiv.org/abs/2503.11945}, 
}

% \appendix
% \input{Sections/supplement}








\end{document}